\documentclass[10pt, conference, letterpaper]{IEEEtran}
\IEEEoverridecommandlockouts
\usepackage{amsmath,amssymb,amsfonts}
\usepackage{algorithmic}
\usepackage{graphicx}
\usepackage{textcomp}
\usepackage{comment}
\usepackage{cite}

\usepackage[T1]{fontenc}
\usepackage[utf8]{inputenc}
\usepackage{tgtermes}
\usepackage{multirow}
\usepackage{stfloats}

\usepackage[ruled,vlined,linesnumbered]{algorithm2e}
\usepackage{subcaption}
\usepackage{float}
\usepackage{booktabs}

\usepackage{xfrac}
\allowdisplaybreaks

\usepackage{caption} 
\usepackage[letterspace=-5]{microtype}

\usepackage[shortlabels]{enumitem}

\usepackage{array}
\newcommand{\PreserveBackslash}[1]{\let\temp=\\#1\let\\=\temp}
\newcolumntype{C}[1]{>{\PreserveBackslash\centering}p{#1}}
\newcolumntype{R}[1]{>{\PreserveBackslash\raggedleft}p{#1}}
\newcolumntype{L}[1]{>{\PreserveBackslash\raggedright}p{#1}}

\usepackage{amsthm}

\newtheorem{assumption}{Assumption}
\newtheorem{theorem}{Theorem}
\newtheorem{corollary}{Corollary}

\usepackage[dvipsnames]{xcolor}

\newcommand{\Identity}{{\rm I\kern-.2em l}}
\newcommand{\Expect}{\mathbb{E}}
\newcommand{\Expectbracket}[1]{\mathbb{E}\left[ #1 \right]}
\newcommand{\Expectcond}[2]{\mathbb{E}\left[\left. #1 \right| #2 \right]}
\newcommand{\x}{\mathbf{x}}
\newcommand{\y}{\mathbf{y}}
\newcommand{\g}{\mathbf{g}}

\newcommand{\normsq}[1]{\left\Vert #1 \right\Vert^2}
\newcommand{\innerprod}[1]{\left\langle #1 \right\rangle}
\newcommand{\Pmax}{P_\text{max}}
\newcommand{\iid}{i.i.d.}

\def\BibTeX{{\rm B\kern-.05em{\sc i\kern-.025em b}\kern-.08em
    T\kern-.1667em\lower.7ex\hbox{E}\kern-.125emX}}
\begin{document}

\title{Communication-Efficient Device Scheduling for Federated Learning Using Stochastic Optimization 
}

\author{
\IEEEauthorblockN{Jake Perazzone\IEEEauthorrefmark{1}, Shiqiang Wang\IEEEauthorrefmark{2}, Mingyue Ji\IEEEauthorrefmark{3}, Kevin S. Chan\IEEEauthorrefmark{1}}
\IEEEauthorblockA{
\IEEEauthorrefmark{1}Army Research Laboratory, Adelphi, MD, USA. Email: \{jake.b.perazzone.civ; kevin.s.chan.civ\}@army.mil, \\
\IEEEauthorrefmark{2}IBM T. J. Watson Research Center, Yorktown Heights, NY, USA. Email: wangshiq@us.ibm.com\\
\IEEEauthorrefmark{3}\textls{Department of Electrical \& Computer Engineering, University of Utah, Salt Lake City, UT, USA. Email: mingyue.ji@utah.edu}
}
\thanks{This research was partly sponsored by the U.S. Army Research Laboratory and the U.K. Ministry of Defence under Agreement Number W911NF-16-3-0001. The views and conclusions contained in this document are those of the authors and should not be interpreted as representing the official policies, either expressed or implied, of the U.S. Army Research Laboratory, the U.S. Government, the U.K. Ministry of Defence or the U.K. Government. The U.S. and U.K. Governments are authorized to reproduce and distribute reprints for Government purposes notwithstanding any copyright notation hereon.}
\vspace{-0.15in}
}

\maketitle

\begin{abstract}
Federated learning (FL) is a useful tool in distributed machine learning that utilizes users' local datasets in a privacy-preserving manner. When deploying FL in a constrained wireless environment; however, training models in a time-efficient manner can be a challenging task due to intermittent connectivity of devices, heterogeneous connection quality, and non-i.i.d. data. In this paper, we provide a novel convergence analysis of non-convex loss functions using FL on both i.i.d. and non-i.i.d. datasets with arbitrary device selection probabilities for each round. Then, using the derived convergence bound, we use stochastic optimization to develop a new client selection and power allocation algorithm that minimizes a function of the convergence bound and the average communication time under a transmit power constraint. We find an analytical solution to the minimization problem. One key feature of the algorithm is that knowledge of the channel statistics is not required and only the instantaneous channel state information needs to be known. Using the FEMNIST and CIFAR-10 datasets, we show through simulations that the communication time can be significantly decreased using our algorithm, compared to uniformly random participation.
%
\end{abstract}


\section{Introduction}
Federated learning (FL) is a valuable machine learning (ML) tool that enables distributed training of neural network models without centralized data by utilizing computation at several distributed learners who use their own local datasets.
Model training is accomplished through a collaborative procedure in which the participating learners are sent the current model and then they each separately perform updates via stochastic gradient descent (SGD) using their own locally collected datasets.
After a set number of local iterations, the participants send their updated model weights to an aggregator who updates the global model, typically through simple averaging of each participant's update, as in \emph{FedAvg}~\cite{mcmahan2017communication}.
The process then repeats by sending out the updated global model to all learners participating in the next round and continues until a satisfactory model is obtained.
A block diagram of the uplink in a wireless network running FL can be found in Figure \ref{fig:blockDiagram} where each learner $n$ is a device that has its own independent channel to the aggregator with fading parameter $h_n(t)$.

\begin{figure}
    \centering
    \includegraphics[width=.7\linewidth]{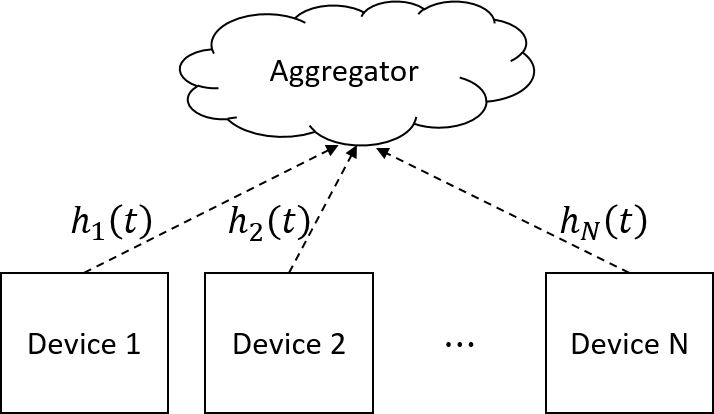}
    \caption{Block diagram of the uplink communication in federated learning over a wireless network.}
    \label{fig:blockDiagram}
\end{figure}

One of the major advantages of the FL training process is that user privacy is preserved since the users' data never leaves their device.
This allows the end user to take part in training and ultimately obtain better ML models without fear of revealing their private data.
The orchestration of FL over large-scale wireless networks, though, has proved to be a challenging task since the amount of communication required to converge to an acceptable model creates a large bottleneck in the process. 
This is particularly evident in dynamic mobile edge computing (MEC) environments where poor channel quality and intermittent connectivity can completely derail training.
For example, if a device loses connection to the server and does not participate in training for an extended period of time, the global model will begin to shift away from their locally optimal model which will negatively affect convergence until they rejoin.
If many devices are absent for extended periods of time, the global model will converge very slowly, or possibly not at all depending on the degree of heterogeneity of the data.
Additionally, if a device is available but has a very bad connection, resources will be wasted if they participate in every round.
Thus, device selection becomes a very important aspect in the management of FL in practice.

In the original FL algorithm, \emph{FedAvg}~\cite{mcmahan2017communication}, clients are selected uniformly at random in each round.
Although this strategy has been shown to converge \cite{li2019convergence,mitra2021achieving}, in practice, devices are not always available for selection due to factors such as energy and time constraints.
Additionally, since this selection policy is agnostic to channel conditions and other factors, it will lead to the consumption of more network resources than necessary.
Thus, a more intelligent approach to device selection is needed to optimize network resource consumption.
However, before designing such an approach, the effect of arbitrary device selection on FL convergence must be understood to ensure convergence to a good model.

In this paper, we derive a novel convergence bound for non-convex loss functions with arbitrary device selection probabilities for each FL round.
Our new upper bound shows that as long as all devices have a non-zero probability of participating in each round, then FL will converge in expectation to a stationary point of the loss function.
We then use the knowledge of how the selection probabilities affect this newly found convergence bound to formulate a stochastic optimization problem that determines the optimal selection probabilities and transmit powers.
The objective function of the problem minimizes a weighted sum of the convergence bound and the time spent for communicating model parameters with a constraint on the peak and time average transmit power.
Communicating the model parameters over many rounds is the major bottleneck of FL.
Therefore, minimizing the communication time is very beneficial in speeding up convergence and minimizing the burden on the network.
The form of the convergence bound and our novel problem formulation allow us to utilize the Lyapunov drift-plus-penalty framework to compute an analytical and distributed solution to the minimization problem with analytical expressions.
A key advantage of our new device selection algorithm is that it is able to make decisions according to current channel conditions without knowledge of the underlying channel statistics.

To show the performance of our algorithm, we run numerous experiments on the CIFAR-10 and FEMNIST datasets to demonstrate the saved communication time using our developed algorithm.
We compare our results to the uniform selection policy of \emph{FedAvg} and show that the time required to reach a target accuracy can be decreased by up to 58\%.
In summary, our main contributions are as follows:
\begin{enumerate}
    \item We derive a new upper bound for convergence of non-convex loss functions using FL with arbitrary selection probabilities.
    \item We formulate a novel stochastic optimization problem that minimizes a weighted sum of the newly found convergence bound and the amount of communication time spent on transmitting parameter updates, while satisfying transmit power constraints.
    \item Using the Lyapunov drift-plus-penalty framework, we derive an analytical and distributed solution to the problem that does not require knowledge of the channel statistics.
    \item We provide experimental results that demonstrate a communication savings of up to 58\% compared to traditional uniform selection strategies.
\end{enumerate}

The rest of the paper is organized as follows.
First, we present some related work in Section \ref{sec:relatedWorks} before formally presenting the FL problem in Section \ref{sec:probForm}.
Then, convergence analysis is provided in \ref{sec:convAnalysis} and the device scheduling policy is developed in Section \ref{sec:lyapunov}.
Finally, we present experimental results in Section~\ref{sec:experiments}.

\section{Related Works}\label{sec:relatedWorks}
Since its introduction in \cite{mcmahan2017communication}, FL has garnered a lot of attention in both industry and academia with a major focus on providing privacy guarantees \cite{mothukuri2021survey,geyer2017differentially,truex2019hybrid}, characterizing convergence \cite{li2019convergence,li2020federated,mitra2021achieving,wang2019adaptive}, and enhancing communication efficiency \cite{konevcny2016federated} through strategies such as model compression via sparsification \cite{han2020adaptive,sattler2019robust} and quantization \cite{alistarh2017qsgd,albasyoni2020optimal}.
One of the biggest challenges with implementing FL at scale is the heterogeneity that is present in both the system and the data.
An alternative way to address communication efficiency and combat system heterogeneity is through device scheduling, or client selection. 
Doing this naively, however, can lead to suboptimal models due to the skew introduced by the heterogeneous, or non-\iid, data at the devices which is why we design our selection process based on its effect on convergence.

One of the first works specifically targeting the problem is \cite{nishio2019client} where the FL training process is modeled in a MEC network with a wide variety of devices.
The scheduling procedure presented was designed to speed up convergence by having as many devices as possible to participate in each round within a desired time window.
The presented strategy, however, does not consider its effect on convergence and results in much poorer performance for non-\iid~datasets as indicated in their results. 
Some other empirical studies with similar approaches include \cite{ribero2020communication,goetz2019active}, but also do not consider or derive convergence bounds for their selection strategies.

Some later work began to include analysis of the convergence of FL with device selection.
Among these is \cite{yang2019scheduling}, but only the convergence of simple linear regression loss is considered. 
In \cite{cho2020client}, the authors analyze the convergence of strongly convex loss functions, but unfortunately their bound introduces a non-vanishing term and thus their strategy is not guaranteed to converge to a stationary point of the loss function. 
Both \cite{ren2020scheduling} and \cite{ruan2021towards} consider convergence, but only for strongly convex loss functions. 
Convergence results for non-convex loss functions with partial device participation have also been presented in \cite{li2020federated,karimireddy2020scaffold,yang2021achieving}, but they only consider the case where devices are chosen uniformly at random with or without replacement and do not allow for arbitrary selection probabilities.
Finally, \cite{gu2021fast} considers arbitrary probabilities for each device, but these probabilities are held constant throughout training and are not reflected in the parameter aggregation weights. 
Additionally, in \cite{gu2021fast}, all devices must participate in the first round for convergence.
We improve upon these results by considering non-convex loss functions and derive a bound with no non-vanishing term under the condition that all devices have an arbitrary non-zero probability of participating in each round.

Some works \cite{chen2020joint,yang2020energy} develop frameworks that jointly optimize convergence and communication over wireless networks.
Similarly to our approach, they derive a convergence bound and then minimize it by finding the optimal parameter values.
For example, in \cite{chen2020joint}, the FL loss is minimized while meeting the delay and energy consumption requirements via power allocation, user selection, and resource block allocation.
Both papers, however, make the unrealistic assumption that the channel remains constant throughout the training process which we do not assume here.

In \cite{wadu2020federated}, stochastic optimization is used to determine an optimal scheduling and resource block policy that simultaneously minimizes the FL loss function and CSI uncertainties.
The loss function considered, though, is simple linear regression and does not readily apply to neural network models.
Stochastic optimization is also considered for FL in \cite{huang2020efficiency} and \cite{zhou2020cefl}, but not to design an optimal device selection policy that guarantees convergence of non-convex loss functions like we do here. 

\section{Problem Formulation}\label{sec:probForm}
We now explain the FL problem in more detail.
Consider a system with $N$ clients, where each client $n$ has a possibly non-convex local objective $f_n(\x)$ with parameter $\x \in \mathbb{R}^d$.
We would like to solve the following finite-sum problem:
\begin{equation}\label{eqn:objFunction}
    \min_\x f(\x) := \frac{1}{N}\sum_{n=1}^N f_n(\x) .
\end{equation}

To solve \eqref{eqn:objFunction}, we follow the FL paradigm and perform a slightly modified version of \emph{FedAvg} \cite{mcmahan2017communication}, where instead of uniform sampling with/without replacement, each client has an arbitrary probability of being selected to participate in a given round, denoted as $q_n^t$ for device $n$ at round $t$.
This modification allows us to adjust the probability of selection for the participating devices based on system dynamics including channel conditions.
Additionally, analyzing \emph{FedAvg} with arbitrary probabilities allows us to observe the effect that a scheduling policy that controls these probabilities has on convergence in order to design a better policy.
The modified algorithm can be found in Algorithm~\ref{alg:fedavg}.

We let $\Identity_t^n \in \{0,1\}$ be a random variable to denote whether client $n$ is sampled in round $t$ such that $q^t_n := \Pr\{\Identity_t^n = 1 \}$ . 
Next, denote $I$ as the synchronization interval, or the number of local SGD updates performed by a device before aggregation, and denote $\g_n(\x)$ as the stochastic gradient of $f_n(\x)$.
We denote the learning rate as $\gamma >0$ and the number of total rounds as $T$. 
Note that this algorithm is logically equivalent to one where only the participating clients receive the model updates and compute the gradient updates.
Notice that each device's aggregation weight is inversely proportional to their probability of being selected.
This ensures that the gradient updates remain unbiased.
Intuitively, it ensures that devices with low participation can still have sufficient influence on the global model when they do participate.


\begin{algorithm}[h]
 \caption{FedAvg with client sampling}
 \label{alg:fedavg}
\KwIn{$\gamma$, $\x_0$, $I$, $T$, $\{q_n^t\}$}
\KwOut{$\{\x_t\}$}

\For{$t \leftarrow 0, \ldots, T-1$}{
    Sample $\Identity_n^t \sim q_n^t, \forall n$;
    
    \For{$n \leftarrow 1,\ldots,N$ in parallel}{
        $\y^n_{t,0} \leftarrow \x_t$;
        
        \For{$i \leftarrow 0, \ldots, I-1$}{
            $\y^n_{t,i+1} \leftarrow \y^n_{t,i} - \gamma \g_n(\y^n_{t,i})$;
        }
        
    }
    
    $\x_{t+1} \leftarrow \x_t + \frac{1}{N}\sum_{n=1}^N \frac{\Identity_n^t}{q_n^t} \left(\y^n_{t,I}-\y^n_{t,0}\right) $; \, \tcp{global parameter update}

}
\end{algorithm}

\section{Convergence Analysis}\label{sec:convAnalysis}
In this section, we prove an upper bound on the convergence of \eqref{eqn:objFunction} using Algorithm~\ref{alg:fedavg} for non-convex loss functions.
We will assume that $\Identity_n^t$ and $\Identity_{n'}^t$ are independent for $n \neq n'$ and that the randomness in client sampling is independent of SGD noise, so that $\Identity_n^t$ and $\g_n$ are independent.
We then make the following assumptions on the local loss functions, which are common in the convergence analysis literature.
\begin{assumption}[$L$-smoothness]\label{asmp:Lsmooth}
    \begin{align}
        \Vert \nabla f_n(\y_1) - \nabla f_n(\y_2) \Vert &\leq L \Vert \y_1 - \y_2 \Vert
    \end{align}
    for any $y_1$, $y_2$ and some $L>0$.
\end{assumption}
\begin{assumption}[Unbiased stochastic gradients]\label{asmp:unbiased}
    \begin{align}
        \Expectcond{\g_n(\y)}{\y}&=\nabla f_n(\y) .
    \end{align}
    for any $\y$
\end{assumption}

Now, we state our novel convergence theorem in Theorem~\ref{thm:1}.
\begin{theorem}\label{thm:1}
Let Assumptions \ref{asmp:Lsmooth} and \ref{asmp:unbiased} hold with $\gamma$, $T$, $I$, $N$, and $q_n^t$ defined as above. Then, Algorithm~\ref{alg:fedavg} satisfies
\begin{align}\label{eqn:thm1}
    \frac{1}{T}\sum_{t=0}^{T-1}&\Expectbracket{\normsq{\nabla f(\x_t)}} \nonumber\\
    &\leq \frac{2\left(f(\x_0) - f^*\right)}{\gamma TI}\nonumber\\
    & \quad + \frac{\gamma^2 L^2(I-1)}{TIN}\sum_{t=0}^{T-1}\sum_{n=1}^N\sum_{i=0}^{I-1}\sum_{j=0}^{i-1}\Expectbracket{\normsq{ \g_n(\y^n_{t,j})}} \nonumber\\
    &\quad + \frac{\gamma L}{TN}  \sum_{t=0}^{T-1}\sum_{n=1}^N \frac{1}{q_n^t} \sum_{i=0}^{I-1}  \Expectbracket{\normsq{  \g_n(\y^n_{t,i})}}
\end{align}
where $f^*$ represents the optimal solution to \eqref{eqn:objFunction}.
\end{theorem}
\begin{proof}
The proof can be found in Appendix \ref{sec:appendix}.
\end{proof}
Furthermore, by making an assumption of uniformly bounded stochastic gradients, we can simplify the bound.
\begin{assumption}[Bounded stochastic gradients]\label{asmp:boundedGrad}
    \begin{align}
        \Expectbracket{\normsq{\g_n(\y)}} \leq G^2, \forall \y,n
    \end{align}
    for some $G > 0$.
\end{assumption}

\begin{corollary}\label{cor:bound}
If Assumption \ref{asmp:boundedGrad} holds, then the bound \eqref{eqn:thm1} becomes
\begin{align}\label{eqn:convBound}
    \frac{1}{T}\sum_{t=0}^{T-1}\Expectbracket{\normsq{\nabla f(\x_t)}} 
    &\leq \frac{2\left(f(\x_0) - f^*\right)}{\gamma TI}+ \gamma^2 L^2(I-1)^2G^2 \nonumber \\
    & \quad + \frac{\gamma LIG^2}{TN}  \sum_{t=0}^{T-1}\sum_{n=1}^N \frac{1}{q_n^t}.
\end{align}
\end{corollary}
\begin{proof}
The proof is a simple application of Assumption \ref{asmp:boundedGrad} and Theorem \ref{thm:1}.
\end{proof}
If we set $\gamma=\frac{1}{\sqrt{T}}$, 
we can guarantee a convergence rate of $\mathcal{O}\left(\frac{1}{\sqrt{T}}\right)$ to a stationary point of $f(\x)$. 
The third term in the bound shows the effect of arbitrary client sampling and follows with the intuition that the more often devices participate, the less iterations will be required to converge.
The bound can be minimized by choosing a selection strategy that minimizes the time average $\frac{1}{TN} \sum_{t=0}^{T-1}\sum_{n=1}^N \frac{1}{q_n^t}$.
While it has a trivial minimum at $q_n^t=1$ for all $n$ and $t$, i.e., full participation, it is impractical to assume that every device can or will participate in every round due to lack of network resources and the amount of time it would take to receive updates from every device.
Now with a convergence bound that is a function of device selection probability $q_n^t$, we can design an optimization problem that properly considers the effect that $q_n^t$ has on convergence in addition to minimizing communication overhead.

\section{Communication-Efficient Scheduling Policy}\label{sec:lyapunov}
In this section, we formulate a novel stochastic optimization problem that chooses the selection probabilities $q_n^t$ and transmit powers $P_n(t)$ in each round, to minimize a function of the convergence bound and communication overhead.
More specifically, we minimize a weighted sum of the last term in \eqref{eqn:convBound} and the average time spent communicating over the channel while satisfying transmission power constraints. 
Since communicating the model parameters over many rounds is the major bottleneck of FL, minimizing the communication time is very beneficial in speeding up convergence time and minimizing the burden on the network.
Our formulation allows for the application of the Lyapunov drift-plus-penalty framework which leads to an analytical solution that does not require knowledge of the exact dynamics or statistics of the channel; only the instantaneous channel state information (CSI) is needed.
The solution can also be computed in a distributed fashion in which each device can determine its own selection probability, and since selection is done independently, each device can notify the aggregator when it should be selected.

We consider a simple wireless network model where all devices are able to communicate with the aggregator and must take turns in transmitting their parameters as in time-division multiple access (TDMA).
For simplicity, we only consider the uplink channel, since the downlink is a broadcast by the aggregator to all the devices that takes much less time.
At each round $t$, the devices receive information about their current CSI in the form of channel gain $|h_n(t)|^2$ and noise power $N_0$.
The algorithm then uses this information to determine each device's probability of selection $q_n^t$ and transmission power $P_n^t$ for that round.
Additionally, the transmission power is subject to both a peak power constraint $P_\text{max}$ and time average constraint $\bar{P}_n$.

We formulate the problem as
\begin{align}\label{eqn:minProb}
    \min_{\{q_n^t\},\{P_n(t)\}} \quad & \lim_{T\rightarrow\infty} \frac{1}{T} \sum_{t=0}^{T-1} \Expectbracket{y_0(t)}\\
    \text{s.t.} \quad & \lim_{T\rightarrow\infty} \frac{1}{T} \sum_{t=0}^{T-1} \Expectbracket{P_n(t) q_n^t} \leq \bar{P}_n, \ \forall n=1,\ldots,N \nonumber\\
    & 0 \leq P_n(t) \leq P_\text{max}, \ n=1,\ldots,N \nonumber\\
    & q_n^t \in (0,1] \nonumber
\end{align}
where 
\begin{align}
    y_0(t) :=  \sum_{n=1}^{N} \left(\frac{1}{N q_n^t} + \lambda\!\cdot\! \frac{\ell\, q_n^t}{B\log_2\left( 1\!+\!|h_n(t)|^2 \frac{P_n(t)}{N_0} \right)}\right),
\end{align}
$\ell$ is the number of bits required to represent the model, $B$ is the bandwidth of the communication channel, $\log_2(\cdot)$ denotes the base 2 logarithm, and $\lambda > 0$ is a tuneable parameter that controls the trade-off between minimizing the convergence bound and the sum transmission time.
The first term in the objective is straightforwardly taken from \eqref{eqn:convBound} while the second term represents the minimum expected amount of time it takes to transmit the model given $q_n^t$.
The denominator of the second term is the channel capacity and although the communication rate in practice is not truly equal to the capacity, it gives us a communication time lower bound that indicates how channel gain and transmission power affect communication times.

The novelty of our convergence bound and formulation comes from the fact that both the effect of $q_n^t$ on convergence and the additional term we add to minimize communication time are in the form of a time average.
This allows us to apply standard theorems from the Lyapunov stochastic optimization framework \cite{neely2010stochastic} to reformulate \eqref{eqn:minProb} into a form that we solve analytically.
While the framework specializes in stabilizing queues in stochastic networks and we have no such queues here, the framework allows us to convert our transmission power constraint into a set of \emph{virtual} queues and apply the Lyapunov convergence theorem to our problem.
The practical implications of the virtual queues will be explored at the end of this section and its effect will be further illustrated in Section \ref{sec:experiments}.
So, to put our optimization problem into the Lyapunov drift-plus-penalty framework and using standard notation, we turn the constraint into a virtual queue $Z_n(t)$ for each client $n$ such that
\begin{align}\label{eqn:virtualQUpdate}
    Z_n(t+1) = \max [Z_n(t)+y_n(t),0],
\end{align}
where \vspace{-0.1in}
\begin{align}
    y_n(t) = P_n(t) q_n^t - \bar{P}_n .
\end{align}
Since we have no actual queues, the Lyapunov function is 
\begin{align}
    L(\mathbf{\Theta}(t)) := \frac{1}{2} \sum_{n=1}^N Z_n(t)^2
\end{align}
where $\mathbf{\Theta}(t)$ represents the current queue states, which in this case, is just $\{Z_n(t): \forall n \}$.
Next, we define the Lyapunov drift:
\begin{align}
    \Delta(t+1) = L(t+1) - L(t),
\end{align}
where we drop $\mathbf{\Theta}(t)$ for simplicity.
Finally, we have the Lyapunov drift-plus-penalty function that we wish to minimize:
\begin{align}\label{eq:driftpluspenalty}
    \Delta(t) + V \Expectbracket{y_0(t)|\mathbf{\Theta}(t)},
\end{align}
where $V>0$ is another arbitrarily chosen weight that controls the fundamental trade-off between queue stability and optimality of the objective functions.

Now, by utilizing Lemma 4.6 from \cite{neely2010stochastic} and assuming that the random event, i.e., channel gain $|h_n(t)|^2$, is \iid~with respect to $t$, we can upper bound \eqref{eq:driftpluspenalty}:
\begin{align}\label{eqn:driftPlusPenaltyBound}
    \Delta(t) + V \Expectbracket{y_0(t)|\mathbf{\Theta}(t)} &\leq C +  V \Expectbracket{y_0(t)|\mathbf{\Theta}(t)}\nonumber\\
        & \quad+ \sum_{n=1}^N Z_n(t)\Expectbracket{y_n(t)|\mathbf{\Theta}(t)}
\end{align}
where $C>0$ is a constant. 
Next, according to the Min Drift-Plus-Penalty Algorithm, we opportunistically minimize the expectation in the right hand side of \eqref{eqn:driftPlusPenaltyBound} at each time step $t$:
\begin{align}
    \min_{\{q_n^t\},\{P_n(t)\}}\,\,\,  & f(q_n^t,P_n(t)) := V y_0(t) + \sum_{n=1}^N Z_n(t)y_n(t) \label{eqn:minimization}\\
    \text{s.t.}\quad  & 0 \leq P_n(t) \leq P_\text{max}, \quad \forall n=1,\ldots,N \nonumber\\
    & q_n^t \in (0,1] \, . \nonumber
\end{align}

\begin{algorithm}[t]
 \caption{Stochastic client sampling}
 \label{alg:lyapunov}
\KwIn{$h_n(t)$, $N_0$, $\ell$, $B$, $V$, $\lambda$, $P_\text{max}$, $\bar{P}_n$}
\KwOut{$q_n^t$, $P_n(t)$}

$Z_n(0) \leftarrow 0$

$ P_n(0) \leftarrow \Pmax$

$ q_n^0 \leftarrow \min\left\{\max \left\{\sqrt{\frac{B\log_2\left(1+|h_n(t)|^2 \frac{\Pmax}{N_0}\right)}{N\lambda \ell}},0\right\},1\right\}$

\For{$t \leftarrow 1, \ldots, T-1$}{
    \For{$n \leftarrow 1,\ldots,N$ in parallel}{
        Calculate roots via \eqref{eqn:powerOpt} and \eqref{eqn:qOpt}

        \If {$0\leq P_n(t) \leq \Pmax$ and $q_n^t \in (0,1]$
        }{
        Perform Hessian determinant test to ensure minimum
        }
        \Else{
        
        $P_n(t) \leftarrow \Pmax$
        
        $q_n^t \leftarrow \min \{\eqref{eqn:qOpt},1\}$
        
        }
        
        $Z_n(t+1) \leftarrow \max [Z_n(t) + P_n(t) q_n^t - \bar{P}_n,0]$

    }
    
}
\end{algorithm}

Since the objective is an independent sum over $n$, we can perform the minimization separately for each device $n$.
Algorithm \ref{alg:lyapunov} details to process in determining the optimal $P_n(t)$ and $q_n^t$ in each round.
We now present Theorem \ref{thm:Lyapunov} which gives an analytical solution to \eqref{eqn:minimization} that can be computed distributively by the devices.

\begin{theorem}\label{thm:Lyapunov}
    The solution to \eqref{eqn:minimization} is given by Algorithm \ref{alg:lyapunov} where the optimal values for each $n$ is given by either the endpoints, i.e., $P_n^\textnormal{opt}(t)=P_\text{max}$, $q_n^t=1$ or by
    \begin{align} \label{eqn:powerOpt}
        P_n^\textnormal{opt}(t) = \frac{N_0}{|h_n(t)|^2} \left(\frac{A}{4} W_0\left(\sqrt{\frac{A}{4}}\right)^{-2}-1\right)
    \end{align}
    where $A=\frac{V\lambda\ell |h_n(t)|^2 \left(\log(2)\right)^2}{N_0 B Z_n(t)}$ and
    \begin{align} \label{eqn:qOpt}
        q_n^{t,\textnormal{opt}} &=\! \left(\frac{\lambda \ell N}{B\log_2\left(1+|h_n(t)|^2 \frac{P_n^\textnormal{opt}(t)}{N_0}\right)}+ \frac{N}{V} Z_n(t) P_n^\textnormal{opt}(t)\right)^{\!\!-\frac{1}{2}}\!\!\!\!\!\!,
    \end{align}
    where $W_0(\cdot)$ is the principal branch of the Lambert $W$ function.
\end{theorem}
\begin{proof}
    The proof can be found in Appendix \ref{sec:appendix2}
\end{proof}

Theorem 4.8 in \cite{neely2010stochastic} and Theorem \ref{thm:Lyapunov} guarantee that this algorithm satisfies
\begin{align}\label{eqn:optimGap}
    \limsup_{t\rightarrow\infty} \frac{1}{t} \sum_{\tau=0}^{t-1} \Expectbracket{y_0(\tau)} \leq y_o^{\textnormal{opt}} + \frac{C}{V},
\end{align}
where $y_o^{\textnormal{opt}}$ is the minimum of $y_o$.
The theorems also guarantee that the transmit power constraint is satisfied as $t\rightarrow\infty$.
The user-defined parameter $V$ traditionally controls the trade-off between the average queue backlog and the gap from optimality, but since we do not have physical queues in our problem, the trade-off does not exist in the same way.
Instead, $V$ controls the speed of convergence in addition to the optimality gap in~\eqref{eqn:optimGap}.

In \eqref{eqn:qOpt}, we can see that when there is a large virtual queue $Z_n(t)$ or chosen transmit power, the probability of selection is decreased in order to satisfy the transmit power constraint.
In this way, the virtual queue represents how far from the time average constraint we are.
As $V$ is increased, the effect that the current virtual queue has on selection becomes less important and it takes longer to satisfy the average power constraint.
This is also explored experimentally in Section \ref{sec:effectV}.
A large $\lambda$ favors the minimization of communication time rather than the convergence bound which naturally leads to lower $q_n^t$ as seen in \eqref{eqn:qOpt}.
Finally, since the probability calculation is done independently by each device, it can be computed locally without direct orchestration by the aggregator.

\section{Experiments}\label{sec:experiments}
In order to demonstrate the advantages of our device scheduling algorithm, we evaluate it on the CIFAR-10 \cite{krizhevsky2009learning} and \mbox{FEMNIST} \cite{caldas2018leaf} datasets and compare the performance to uniform device sampling in terms of total time for communicating model parameters.
For simplicity, we assume that the computation time is much less than communication time and do not include that in our time measurements.
The FEMNIST dataset is a federated partitioning of the extended MNIST (EMNIST) dataset \cite{cohen2017emnist} that consists of 62 classes of handwritten letters and digits from 3597 different writers.
In the experiments, each device is given data from only one writer in order to simulate a more realistic heterogeneous data environment rather than partitioning by class as is sometimes done, e.g., in \cite{zhao2018federated}.
Therefore, for the FEMNIST dataset, we consider $N=3597$ clients in which we reserve 10\% of the data for testing.
For the CIFAR-10 dataset, on the other hand, we only consider the \iid~case where $N=100$ clients are given a uniform sampling from the 50,000 color images of 10 classes where 10,000 images are reserved for testing.

In both experiments, we train the same convolutional neural network (CNN) as in \cite{wang2019adaptive,han2020adaptive} which has $d=555,178$ parameters for CIFAR-10 and $d=444,062$ parameters for FEMNIST.
Therefore, for Algorithm \ref{alg:lyapunov}, we set $\ell=32d$ since each parameter is represented as a 32 bit floating point number.
We also set the minibatch size to 32, $\gamma$ = 0.01, $I=10$, and $B=22\times 10^{6}$ to simulate WiFi bandwidth.
The power constraints are set to $\bar{P}_n=1$ and $\Pmax=100$ and the noise power is normalized to $N_0=1$.
For the channel model, we assume each device experiences Rayleigh fading such that $|h_n(t)|$ is distributed as a Rayleigh random variable.
In the first set of experiments, we assume that every device has the same Rayleigh parameter, $\sigma=1$, but change to a more heterogeneous setup for the next group of experiments.
Note that our algorithm does not need to know either the parameters or the distribution of the channel gain itself.

In the uniform selection cases, we choose the number of devices to be selected in each round to match the average number of devices selected using our algorithm for different $\lambda$ values.
The average number of devices selected by our algorithm, denoted by $M$, is estimated using the Monte Carlo method. \textit{Note that the optimal number of selected devices by uniform selection is not known in practice; hence, we consider a stronger benchmark here than the commonly used uniform selection method.}
To satisfy the transmit power constraint in \eqref{eqn:minProb} for the uniform case, we set $P_n(t)=\bar{P}_n \cdot \frac{N}{M'}$ for all $n$ and $t$, where $M'$ is the number of devices selected in a given round that is equal to either $\lfloor M\rfloor$ or $\lceil M \rceil$.
We use a moving average with a window size of 500 iterations to smooth the curves for a better viewing experience.

To keep the communications channel realistic, we upper and lower bound the possible values for $|h_n(t)|^2$.
For the upper bound, we set $|h_n(t)|^2< (2^{10}-1)N_0/\bar{P}$ since, in practice, with a very good channel, modern communication system can only go up to 1024-QAM which is 10 bits/s/Hz.
For the lower bound, we set $|h_n(t)|^2< (2^{.25}-1)N_0/\Pmax$ to avoid big outliers that likely will not be chosen by either selection policy and only assume error correction is available at a rate of $.25$ bits/s/hz at the maximum transmit power.
Additionally, in cases where the value of $\lambda$ results in very low selection probabilities, we ensure that at least one device is selected each round by choosing the device with the largest $q_n^t$ if none are chosen during the regular selection process.

\subsection{CIFAR-10 Results}
First, we present our experimental results for the \iid~CIFAR-10 dataset.
The results of our experiments are shown in Figure \ref{fig:CIFAR10}.
In Figures \ref{fig:CIFAR10_IID_test} and \ref{fig:CIFAR10_IID_loss}, we consider a homogeneous network where each device has the same Rayleigh fading parameter, while in Figures \ref{fig:CIFAR10_nonIID_test} and \ref{fig:CIFAR10_nonIID_loss}, we consider a heterogeneous network where the fading parameter is different for each device.
More specifically, in the homogeneous case, we set the Rayleigh fading parameter such that all 100 devices have variance $\sigma=1$.
In the heterogeneous channel case, we set the Rayleigh fading parameter such that 10 devices have $\sigma=0.2$, 40 have $\sigma=0.75$, and 50 have $\sigma=1.2$.
In all four plots, we look at the cases where $\lambda=10$ and $\lambda= 100$ and compare them to uniform selection with $M= 5.99$ and $M=2.5$, respectively, for the homogeneous channel case, and $M= 5.65$ and $M=2.41$, respectively, for the heterogeneous channel case.
In the uniform case, fractional devices are chosen by choosing the floor or ceiling of $M$ with the appropriate probability.
We set $V=1000$ for our algorithm and justify this choice in Section \ref{sec:effectV}.

\begin{figure}
    \centering
    \begin{subfigure}[b]{0.475\linewidth}
        \centering
        \includegraphics[width=1.1\linewidth]{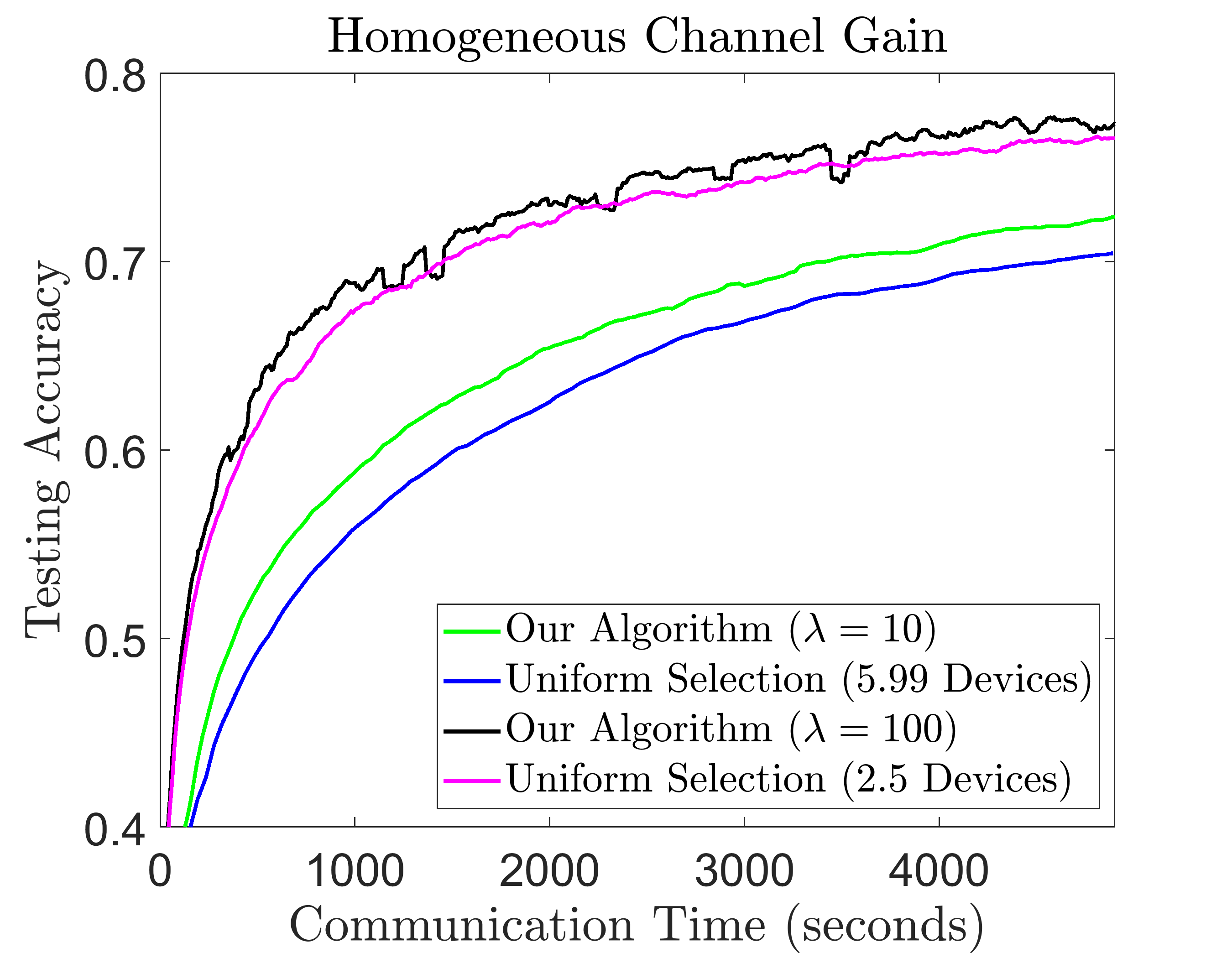}
        \caption{Testing accuracy over time.}
        \label{fig:CIFAR10_IID_test}
    \end{subfigure}
    \hfill
    \begin{subfigure}[b]{0.475\linewidth}
        \centering
        \includegraphics[width=1.1\linewidth]{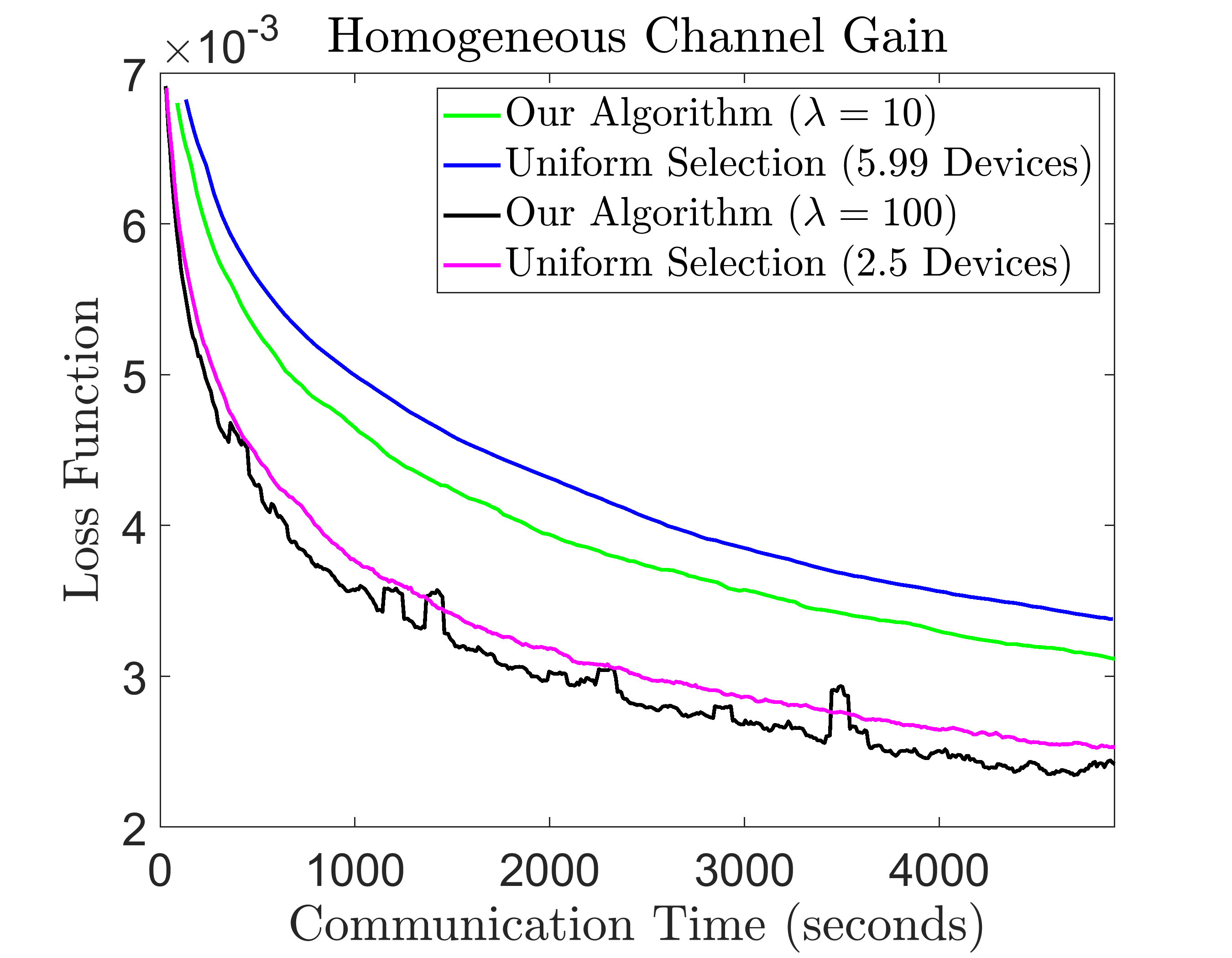}
        \caption{Loss function over time.}
        \label{fig:CIFAR10_IID_loss}
    \end{subfigure}
    \vspace{0.1in}
    
    \begin{subfigure}[b]{0.475\linewidth}
        \centering
        \includegraphics[width=1.1\linewidth]{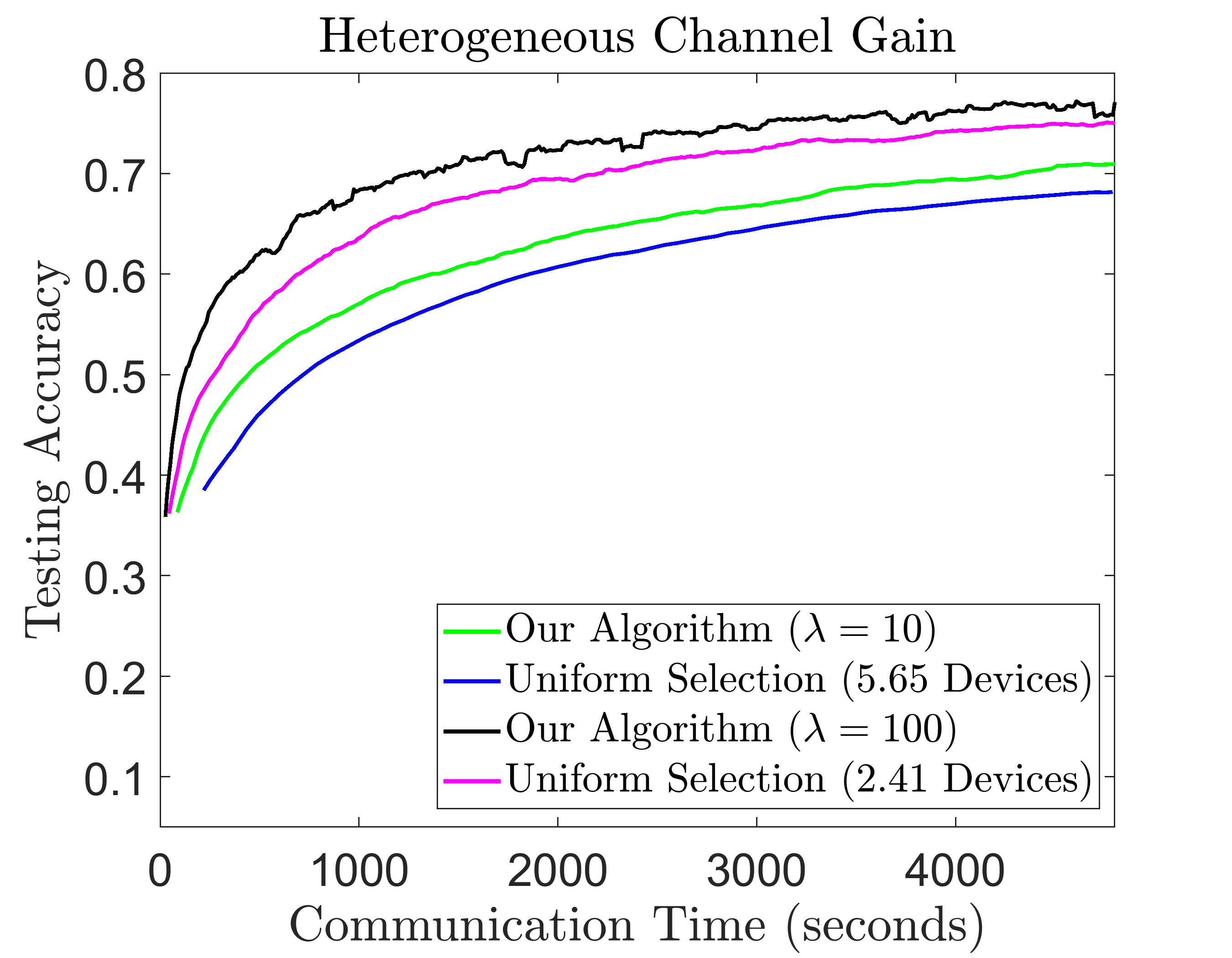}
        \caption{Testing accuracy over time.}
        \label{fig:CIFAR10_nonIID_test}
    \end{subfigure}
    \hfill
    \begin{subfigure}[b]{0.475\linewidth}
        \centering
        \includegraphics[width=1.1\linewidth]{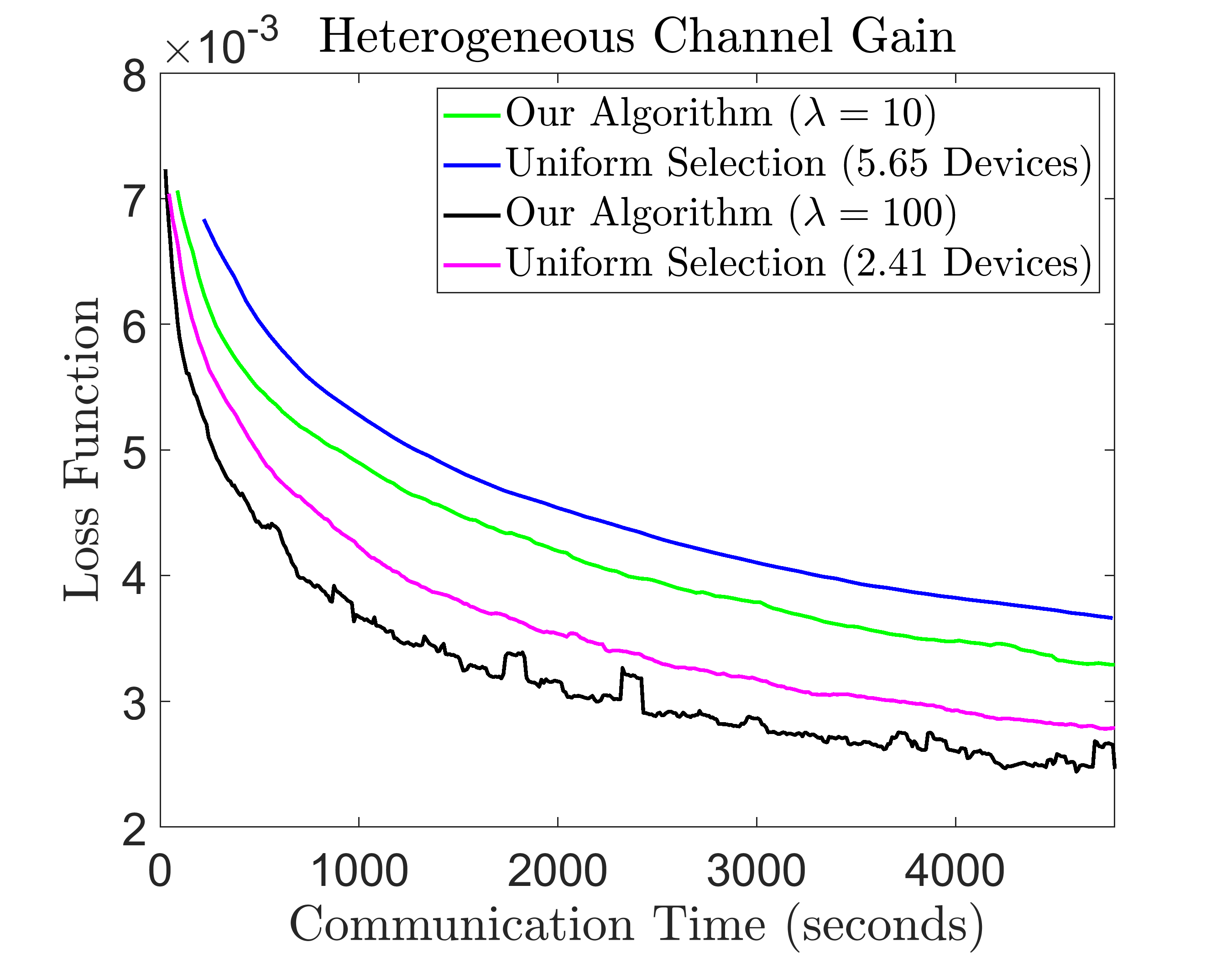}
        \caption{Loss function over time.}
        \label{fig:CIFAR10_nonIID_loss}
    \end{subfigure}
    \caption{Comparison of total communication time for uniform selection vs proposed algorithm on CIFAR-10 dataset.}
    \label{fig:CIFAR10}
\end{figure}

In this \iid~data case, the advantages of our scheme are readily apparent as our selection policy consistently reaches testing accuracy values in less time compared to the uniform equivalent.
The achieved training speed up is more noticeable in the heterogeneous channel case since the algorithm picks the devices with bad channels less often.
When comparing the Figures \ref{fig:CIFAR10_IID_test} and \ref{fig:CIFAR10_nonIID_test}, it is most clear in the $\lambda=100$ case.
For example, in the homogeneous channel case, our algorithm first reaches an accuracy of $0.7$ in $79.2\%$ less time whereas in the heterogeneous channel case, our algorithm first reaches an accuracy of $0.7$ in $58.2\%$ less time which is a larger speed up.

We also note that the reason that the selection schemes that choose fewer devices per round, e.g., the $\lambda=100$ and corresponding uniform cases, appear to converge to a higher accuracy faster is because they are able to complete more iterations in the given time frame.
So, while having fewer devices participate in a round generally results in a poorer quality update due to increased variance, it allows for the local models to be aggregated more quickly and thus can end up resulting in faster convergence in time.
In other words, quantity over quality wins out.
To illustrate the worse per round performance of the fewer device per round regimes, we plot the same results from Figure~\ref{fig:CIFAR10} in Figure~\ref{fig:effectOfLambda}, but versus communication rounds rather than communication time.
We reiterate that larger $\lambda$ means fewer devices chosen per round on average.
It is clear that as $\lambda$ is increased, the testing accuracy converges more slowly per round and oscillates more intensely.
This reveals an interesting unsolved trade-off between the quality versus the speed of global updates in federated learning.
While the scenario and datasets considered here favor faster, lower quality updates, this might not always be the case.
For example, the optimal update policy will depend on things like the communication/channel model used and the computation time required to compute updates.

\begin{figure}
    \centering
    \begin{subfigure}[b]{0.475\linewidth}
        \centering
        \includegraphics[width=1.1\linewidth]{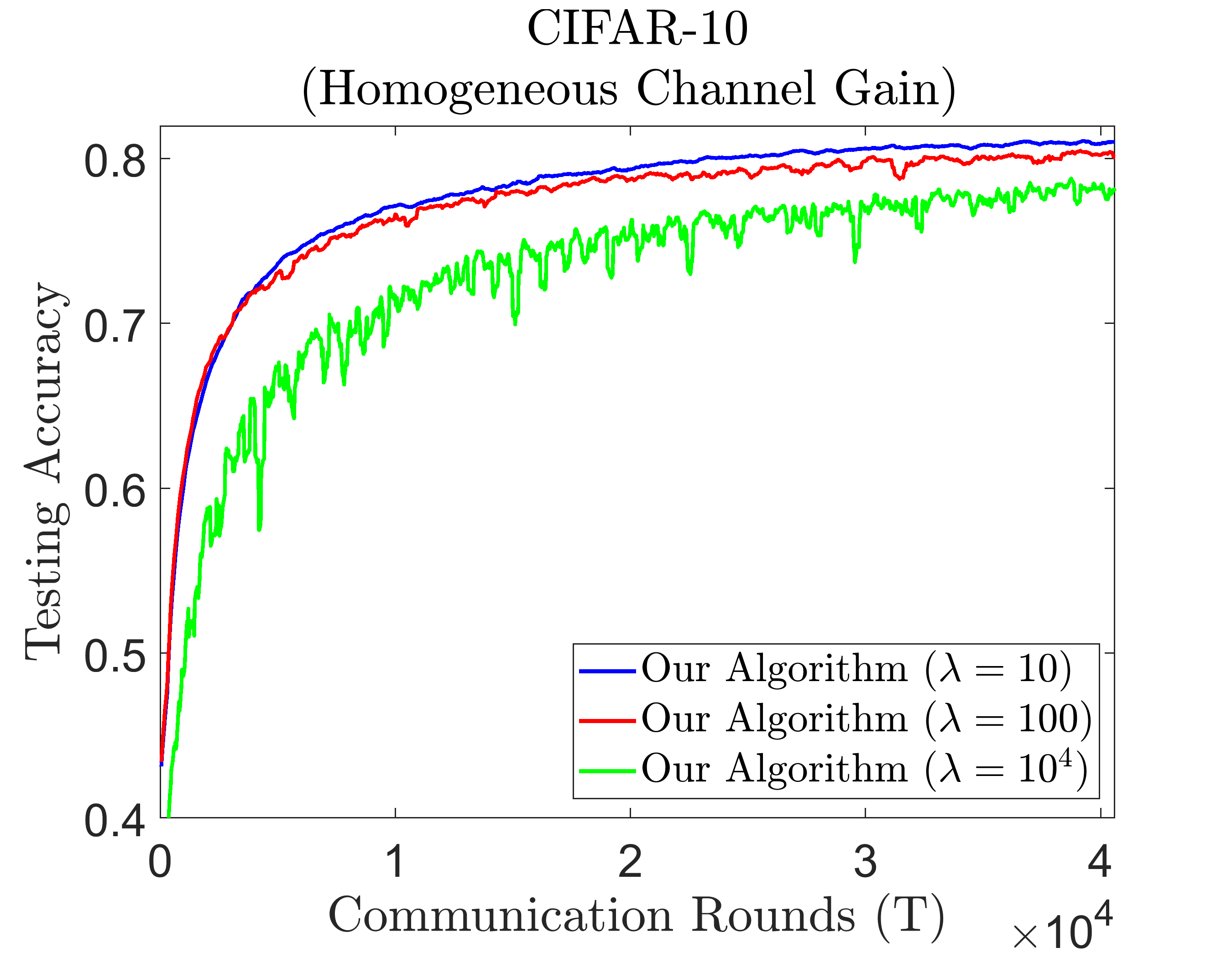}
        \caption{Testing accuracy over rounds.}
        \label{fig:CIFAR10_IID_test_ext}
    \end{subfigure}
    \hfill
    \begin{subfigure}[b]{0.475\linewidth}
        \centering
        \includegraphics[width=1.1\linewidth]{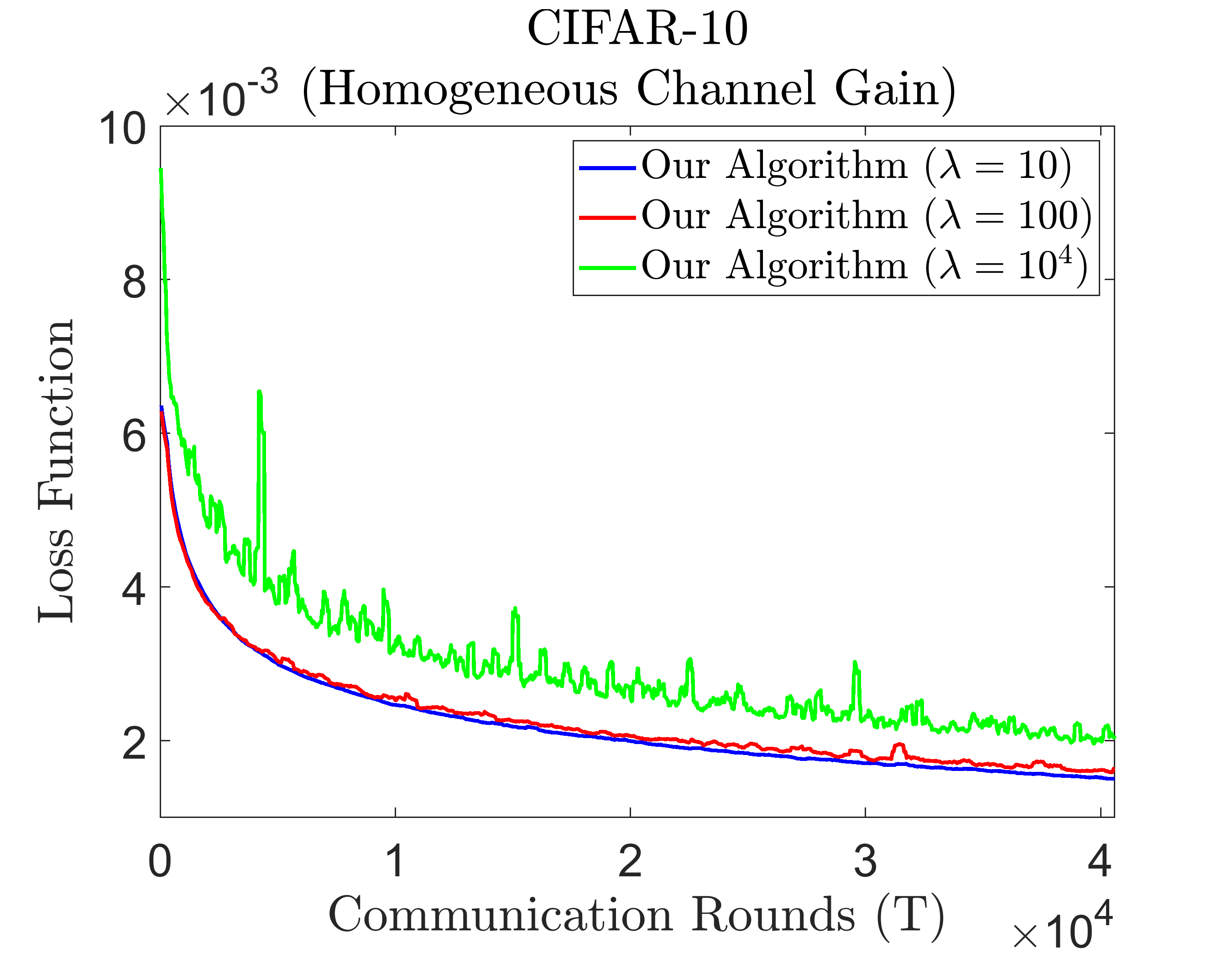}
        \caption{Loss function over rounds.}
        \label{fig:CIFAR10_IID_loss_ext}
    \end{subfigure}
    \caption{Effect of $\lambda$ (CIFAR-10).}
    \label{fig:effectOfLambda}
\end{figure}


\subsection{FEMNIST Results}
In our next experiment, we compare the total communication time for the FEMNIST dataset using uniform sampling versus our algorithm. We set the fading parameters such that for the heterogeneous case, 500 clients have $\sigma=0.2$, 1500 clients have $\sigma=0.75$, and 1597 clients have $\sigma=1.2$, while for the homogeneous case, we set $\sigma=1$ for all devices.
We again set $V=1000$ for our algorithm.
In uniform selection, we set $M= 54.36 $ and $M= 19.4$ devices for $\lambda=10$ and $\lambda=100$, respectively, for the homogeneous case, and we set $M= 52.7 $ and $M= 18.62 $ devices for $\lambda=10$ and $\lambda=100$, respectively, for the heterogeneous case.
The results are shown in Figure \ref{fig:FEMNIST}.

Interestingly, in the homogeneous channel case (Figures \ref{fig:FEMNIST_IID_test} and \ref{fig:FEMNIST_IID_loss}), the two selection strategies perform very similarly with a marginal increase in speed for our algorithm.
This is most likely due to the greater number of devices being chosen and the similarity in channel gain causing the algorithm to choose in such a way that is close to uniform selection.
For the heterogeneous channel gain case, on the other hand, the more varying channel gains causes the algorithm to choose the devices with better channels more frequently.
Since our algorithm guarantees that our algorithm converges even when training on non-\iid~data, the model still converges and benefits from the time saved using our device selection policy.
Another interesting note is that, in the heterogeneous case, the percentage of speed up is better for the $\lambda=10$ case than the $\lambda=100$ case.
For example, the testing accuracy reaches $0.8$ in $69.5\%$ less time in the $\lambda=10$ case compared to its uniform equivalent and $86.2\%$ less time in the $\lambda=100$ case compared to its uniform equivalent.

\begin{figure}
    \centering
    \begin{subfigure}[b]{0.475\linewidth}
        \centering
        \includegraphics[width=1.1\linewidth]{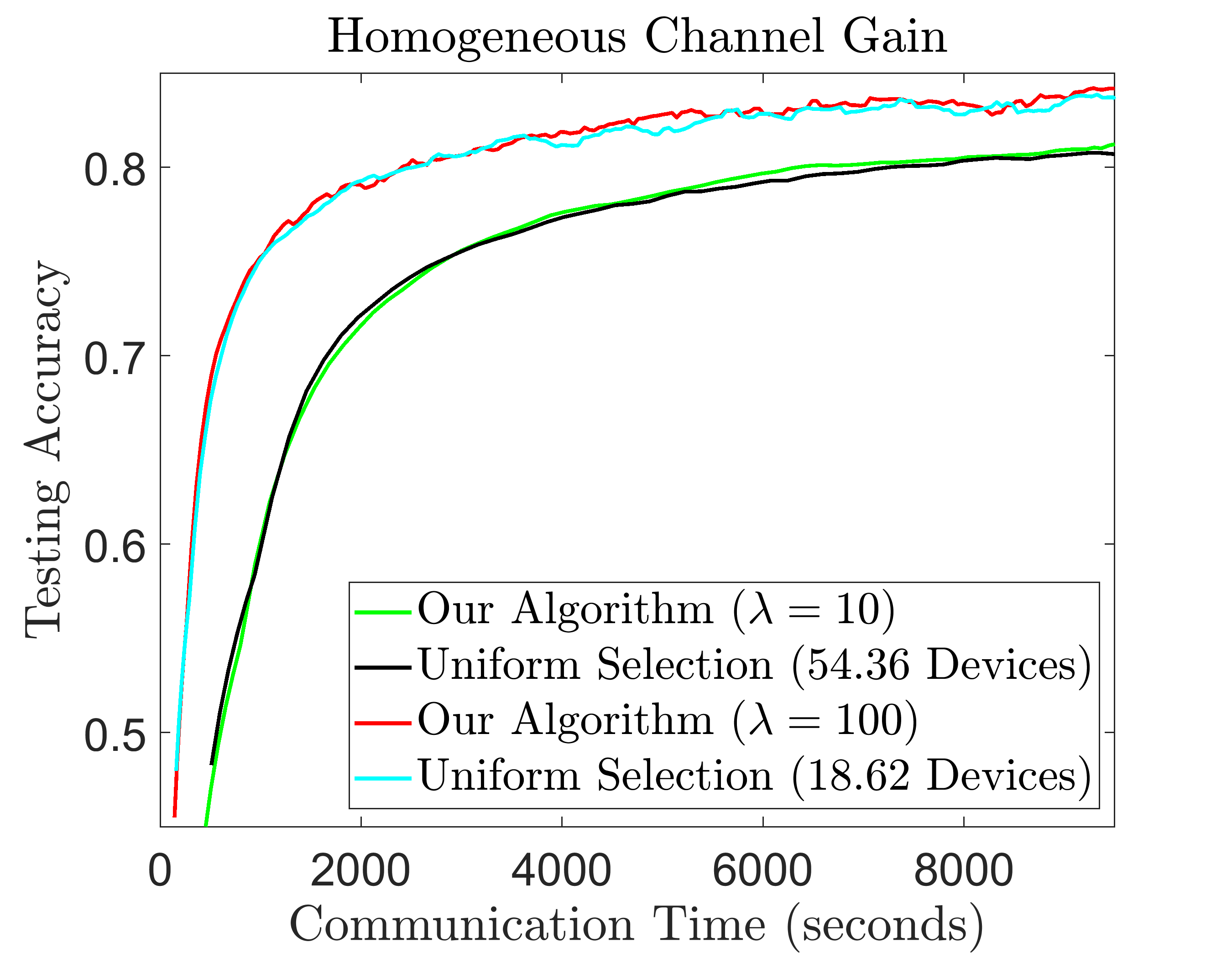}
        \caption{Testing accuracy over time.}
        \label{fig:FEMNIST_IID_test}
    \end{subfigure}
    \hfill
    \begin{subfigure}[b]{0.475\linewidth}
        \centering
        \includegraphics[width=1.1\linewidth]{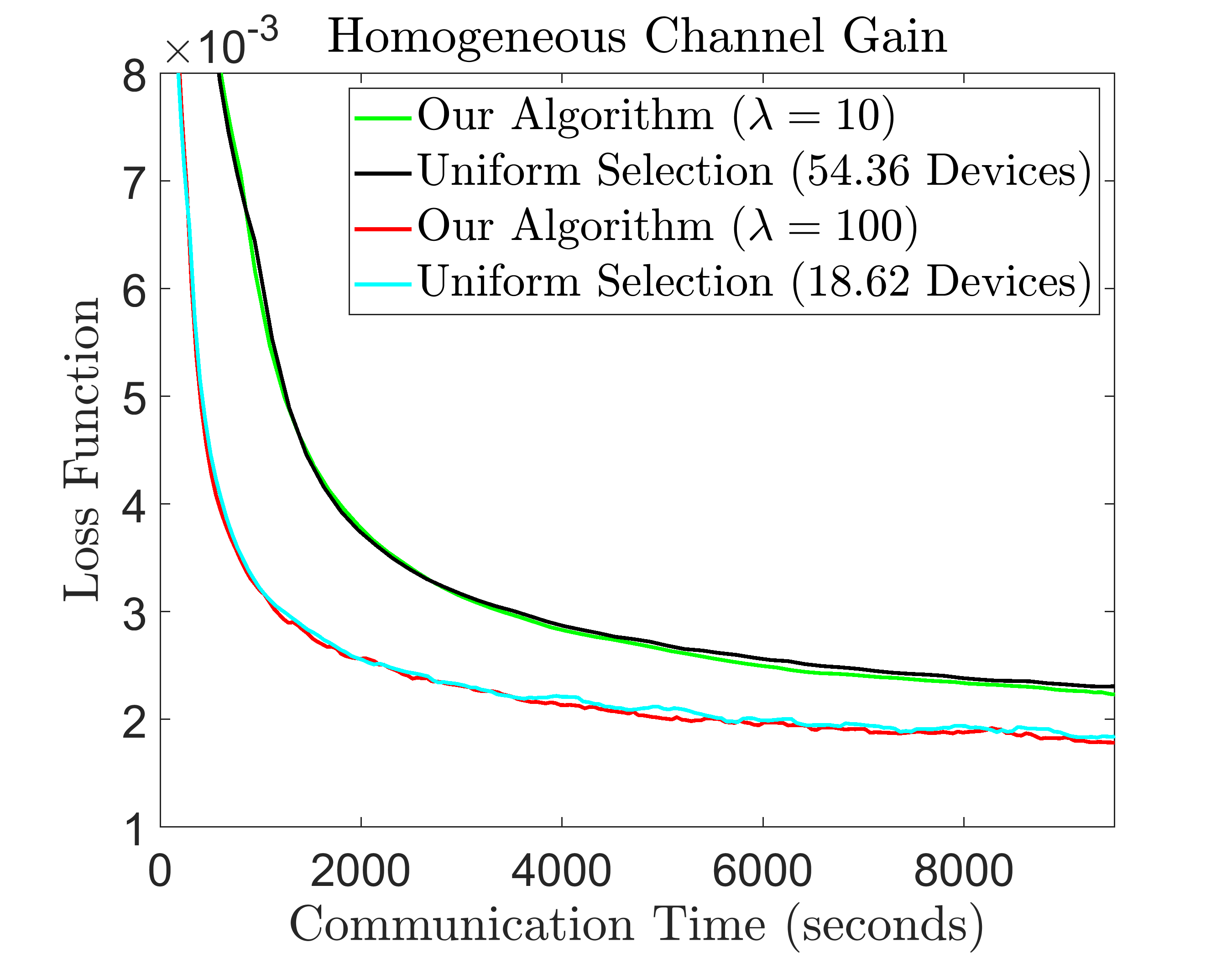}
        \caption{Loss function over time.}
        \label{fig:FEMNIST_IID_loss}
    \end{subfigure}
    \vskip\baselineskip
    \begin{subfigure}[b]{0.475\linewidth}
        \centering
        \includegraphics[width=1.1\linewidth]{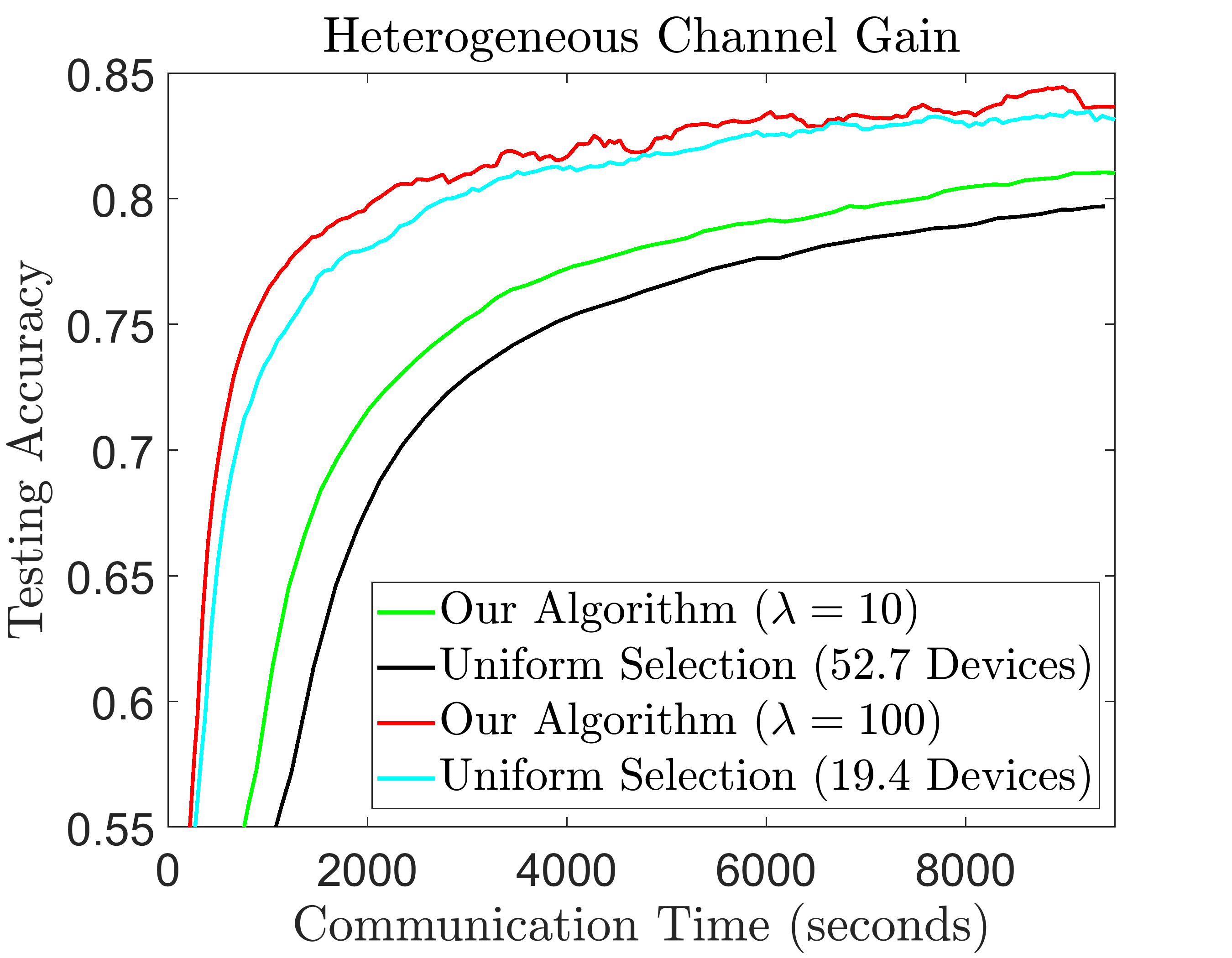}
        \caption{Testing accuracy over time.}
        \label{fig:FEMNIST_nonIID_test}
    \end{subfigure}
    \hfill
    \begin{subfigure}[b]{0.475\linewidth}
        \centering
        \includegraphics[width=1.1\linewidth]{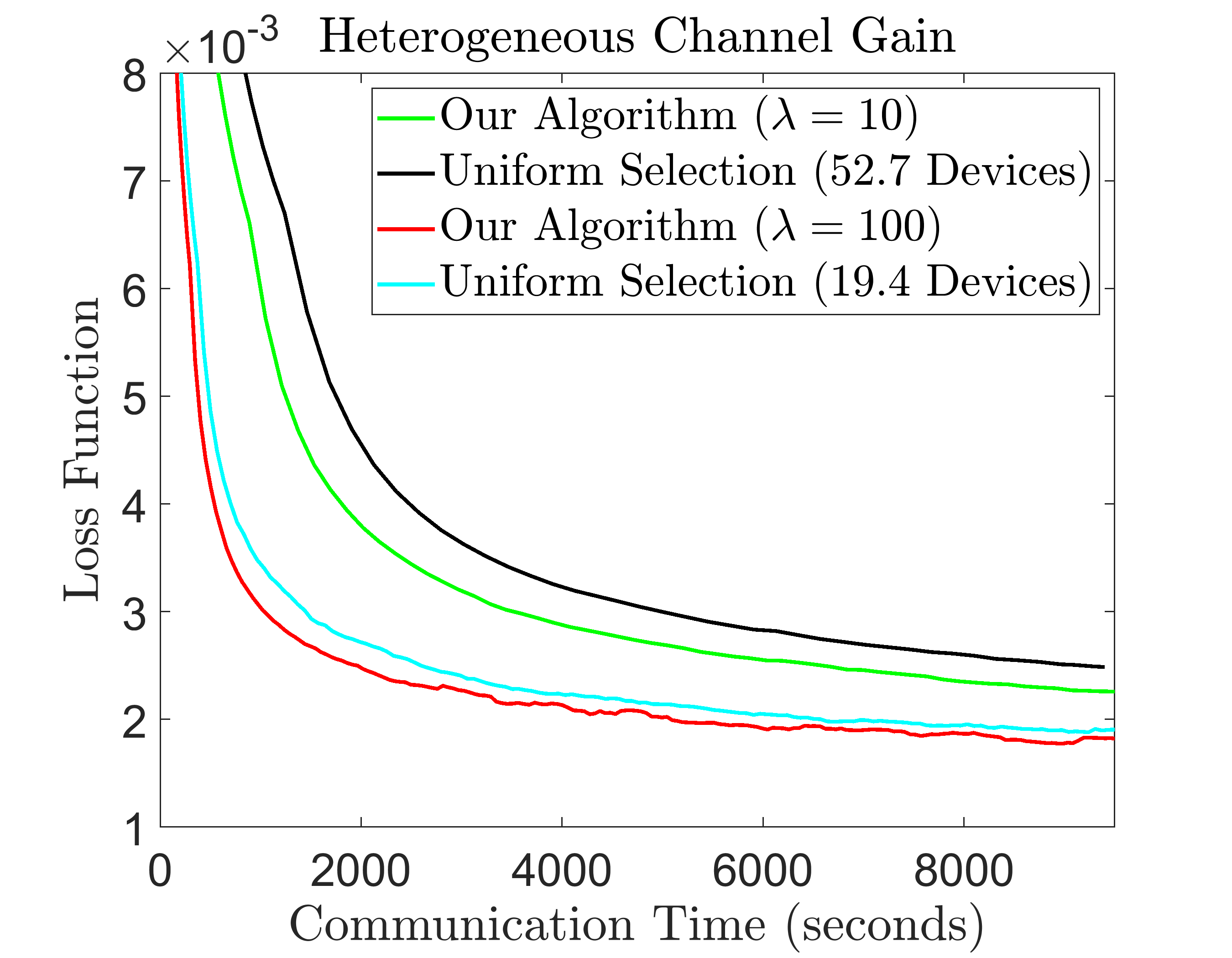}
        \caption{Loss function over time.}
        \label{fig:FEMNIST_nonIID_loss}
    \end{subfigure}
    \caption{Comparison of total communication time for uniform selection vs proposed algorithm on FEMNIST dataset.}
    \label{fig:FEMNIST}
\end{figure}

\subsection{The Effect of $V$}\label{sec:effectV}
In Figure \ref{fig:effectOfV}, we plot expected time average transmit power $\frac{1}{T} \sum_{t=0}^{T-1} {P_n(t) q_n^t}$ over the course of training rounds to show how the parameter $V$ in our algorithm affects the satisfaction of the power constraint.
While large $V$ brings us closer to the optimal values that minimize the weighted sum of the time average from the convergence upper bound and the total communication time, it also takes more rounds for the time average power constraint to be satisfied.
For $V=1$, the constraint is satisfied very quickly and oscillates around $\bar{P}=1$, while for $V=10^{5}$ case, it takes many more rounds to satisfy the constraint.
We also note for comparison purposes that the power allocated in the uniform selection case will always satisfy the constraint by design.
Thus, our algorithm sacrifices not satisfying the constraint initially in finite time in order to make gains in performance, but always satisfies the constraint asymptotically.
Our gains are not solely attributed to this, however.
For the previous experiments, we chose $V=1000$ since it satisfies the constraint in about the same amount of rounds as it takes for the loss function to achieve a desired value.

\begin{figure}
    \centering
    \includegraphics[width=.5\linewidth]{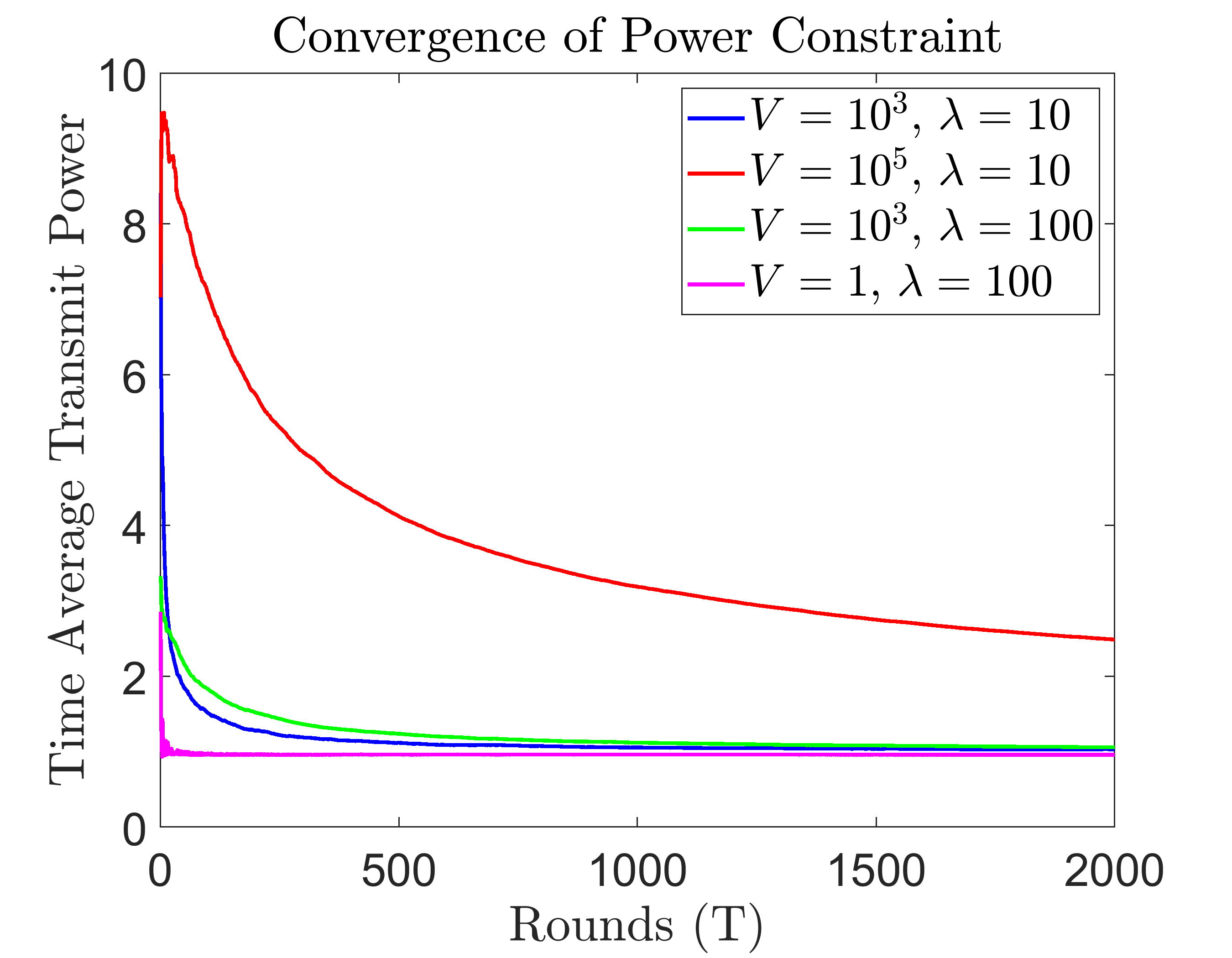}
    \caption{The convergence of the constraint for different values of $V$. The larger the $V$, the more rounds it takes until the constraint is satisfied. Here, the constraint is $\bar{P}_n = 1 $ for all $n$.}
    \label{fig:effectOfV}
    \vspace{-0.1in}
\end{figure}

\section{Conclusions and Future Works}
In this paper, we studied the affect of arbitrary selection probabilities for devices in federated learning and noted the challenge of scheduling devices in a heterogeneous wireless environment.
After deriving a novel convergence bound for non-convex loss functions, we formulated a stochastic optimization problem that minimizes a weighted sum of the derived convergence bound and the total time spent on transmitting the parameter updates under a transmit power constraint.
By using the Lyapunov drift-plus-penalty framework, we developed an algorithm that analytically solves the formulated problem to find the optimal selection probabilities and transmit powers.
Our experimental results showed that even without knowledge of the channel statistics, a significant amount of time can be saved during the FL training procedure using our algorithm.
We used a realistic non-\iid~dataset known as FEMNIST to demonstrate how the algorithm might perform in practice and the results were very promising for heterogeneous wireless environments.
We also showed via the CIFAR-10 dataset that the gains can be even greater when the data is \iid.
Future work may consider multiple access communication schemes and seek to minimize the slowest of the chosen devices since aggregation will ultimately be waiting for the last update.
There is potential for many interesting directions by considering different objective functions 
to focus on different aspects of the FL process.


\appendix

\subsection{Proof of Theorem \ref{thm:1}}\label{sec:appendix}
\textbf{Preliminary inequalities.}
By Jensen's inequality:
\begin{align}
\begin{array}{cc}
     \normsq{ \frac{1}{M}\sum_{m=1}^M \y_m } \leq \frac{1}{M}\sum_{m=1}^M \normsq{\y_m }
\end{array}
     \\
\begin{array}{cc}
    \normsq{ \sum_{m=1}^M \y_m } &\leq M \sum_{m=1}^M \normsq{ \y_m }
\end{array}
\end{align}
By Peter-Paul inequality:
\begin{equation}
    \langle \y_1, \y_2 \rangle \leq \frac{\rho \normsq{ \y_1 }}{2} + \frac{\normsq{ \y_2 }}{2\rho}
\end{equation}
for $\rho > 0$.

First, we note that
\begin{align*}
    \x_{t+1}-\x_t 
    & = \frac{1}{N}\sum_{n=1}^N \frac{\Identity_n^t}{q_n^t} (\y^n_{t,I} - \y^n_{t,0})\\
    &= - \frac{\gamma}{N}\sum_{n=1}^N \frac{\Identity_n^t}{q_n^t} \sum_{i=0}^{I-1}  \g_n(\y^n_{t,i})
\end{align*}
Then, from $L$-smoothness, we have
\begin{align}
    &\Expect[f(\x_{t+1})|\x_t] \nonumber\\
    &\leq f(\x_t) +\innerprod{\nabla f(\x_t), \Expect[\x_{t+1} - \x_{t}|\x_t]} \nonumber \\
    &\qquad + \frac{L}{2} \Expectcond{\normsq{\x_{t+1} - \x_{t}}}{\x_t}\nonumber\\
    &= f(\x_t) - \gamma\innerprod{\nabla f(\x_t), \Expectcond{\frac{1}{N}\sum_{n=1}^N \frac{\Identity_n^t}{q_n^t} \sum_{i=0}^{I-1}  \g_n(\y^n_{t,i})}{\x_t}}\nonumber\\
    & \qquad + \frac{\gamma^2 L}{2N^2} \Expectcond{\normsq{\sum_{n=1}^N \frac{\Identity_n^t}{q_n^t} \sum_{i=0}^{I-1}  \g_n(\y^n_{t,i})}}{\x_t}\nonumber\\
    &\overset{(a)}{=} f(\x_t) - \gamma\innerprod{\nabla f(\x_t), \frac{1}{N}\sum_{n=1}^N  \sum_{i=0}^{I-1}  \Expectcond{\nabla f_n(\y^n_{t,i})}{\x_t}}\nonumber \\
    & \qquad + \frac{\gamma^2 L}{2N^2} \Expectcond{\normsq{\sum_{n=1}^N \frac{\Identity_n^t}{q_n^t} \sum_{i=0}^{I-1}  \g_n(\y^n_{t,i})}}{\x_t}\nonumber\\
    &= f(\x_t) - \gamma  \sum_{i=0}^{I-1}  \Expectcond{\innerprod{\nabla f(\x_t),\frac{1}{N}\sum_{n=1}^N \nabla f_n(\y^n_{t,i})}}{\x_t} \nonumber\\
    &\qquad + \frac{\gamma^2 L}{2N^2} \Expectcond{\normsq{\sum_{n=1}^N \frac{\Identity_n^t}{q_n^t} \sum_{i=0}^{I-1}  \g_n(\y^n_{t,i})}}{\x_t}
    \label{eq:convergence_proof_1}
\end{align}
where (a) uses the independence between $\Identity_n^t$ and $\g_n$, the fact that $\Expectcond{\Identity_n^t}{\x_t} = \Expectbracket{\Identity_n^t} = q_n^t$, and the total expectation $\Expectcond{\g_n(\y^n_{t,i})}{\x_t} = \Expectcond{\Expectcond{\g_n(\y^n_{t,i})}{\y^n_{t,i}, \x_t}}{\x_t} =  \Expectcond{\nabla f_n(\y^n_{t,i})}{\x_t}$.

For the last term, we note that
\begin{align*}
    &\Expectcond{\normsq{\sum_{n=1}^N \frac{\Identity_n^t}{q_n^t} \sum_{i=0}^{I-1}  \g_n(\y^n_{t,i})}}{\x_t} \\
    &\qquad \qquad \leq   N \sum_{n=1}^N  \Expectcond{\normsq{ \frac{\Identity_n^t}{q_n^t} \sum_{i=0}^{I-1}  \g_n(\y^n_{t,i})}}{\x_t} \\
    &\qquad \qquad= N \sum_{n=1}^N \frac{\Expectcond{\Identity_n^t}{\x_t}}{(q_n^t)^2} \Expectcond{\normsq{  \sum_{i=0}^{I-1}  \g_n(\y^n_{t,i})}}{\x_t} \\
    &\qquad \qquad\leq NI \sum_{n=1}^N \frac{q_n^t}{(q_n^t)^2} \sum_{i=0}^{I-1}  \Expectcond{\normsq{  \g_n(\y^n_{t,i})}}{\x_t} \\
    &\qquad \qquad\leq NI \sum_{n=1}^N \frac{1}{q_n^t} \sum_{i=0}^{I-1}  \Expectcond{\normsq{  \g_n(\y^n_{t,i})}}{\x_t}
\end{align*}

Plugging back into (\ref{eq:convergence_proof_1}), we get
\begin{align}
    &\Expect[f(\x_{t+1})|\x_t] \nonumber\\
    &\leq f(\x_t) - \gamma  \sum_{i=0}^{I-1}  \Expectcond{\innerprod{\nabla f(\x_t),\frac{1}{N}\sum_{n=1}^N \nabla f_n(\y^n_{t,i})}}{\x_t} \nonumber\\
    &\qquad + \frac{LI \gamma^2}{2N}  \sum_{n=1}^N \frac{1}{q_n^t}  \sum_{i=0}^{I-1}  \Expectcond{\normsq{  \g_n(\y^n_{t,i})}}{\x_t}
\end{align}
Taking total expectation on both sides, we have
\begin{align}
    \Expectbracket{f(\x_{t+1})} &\leq \Expectbracket{f(\x_t)}\nonumber\\
    &\quad - \gamma   \sum_{i=0}^{I-1}  \Expectbracket{\innerprod{\nabla f(\x_t),\frac{1}{N}\sum_{n=1}^N\nabla f_n(\y^n_{t,i})}}\nonumber\\
    &\quad + \frac{LI \gamma^2}{2N}  \sum_{n=1}^N \frac{1}{q_n^t} \sum_{i=0}^{I-1}  \Expectbracket{\normsq{  \g_n(\y^n_{t,i})}} \label{eq:convergence_proof_2}
\end{align}

Now, note that
\begin{align*}
    &-\gamma\Expectbracket{\innerprod{\nabla f(\x_t),\frac{1}{N}\sum_{n=1}^N\nabla f_n(\y^n_{t,i})}} \\
    &= -\gamma\Expectbracket{\innerprod{\nabla f(\x_t),\frac{1}{N}\sum_{n=1}^N\nabla f_n(\y^n_{t,i}) - \nabla f(\x_t) + \nabla f(\x_t)}} \\
    &= \gamma\Expectbracket{\innerprod{\nabla f(\x_t), \nabla f(\x_t) - \frac{1}{N}\sum_{n=1}^N\nabla f_n(\y^n_{t,i}) }} \nonumber\\
    &\qquad -\gamma \Expectbracket{\innerprod{\nabla f(\x_t),\nabla f(\x_t)}} \\
    &\leq \frac{\gamma}{2}\Expectbracket{\normsq{\nabla f(\x_t)}} + \frac{\gamma}{2}\Expectbracket{\normsq{\nabla f(\x_t) - \frac{1}{N}\sum_{n=1}^N\nabla f_n(\y^n_{t,i})}}\nonumber\\
    &\qquad- \gamma\Expectbracket{\normsq{\nabla f(\x_t)}} \\
    &= \frac{\gamma}{2}\Expectbracket{\normsq{\nabla f(\x_t)}} \nonumber\\
    &\qquad + \frac{\gamma}{2}\Expectbracket{\normsq{ \frac{1}{N}\sum_{n=1}^N \left[\nabla f_n(\x_t) - \nabla f_n(\y^n_{t,i})\right]}}\nonumber\\
    &\qquad - \gamma\Expectbracket{\normsq{\nabla f(\x_t)}} \\
    &\leq  \frac{\gamma}{2N}\sum_{n=1}^N\Expectbracket{\normsq{\nabla f_n(\x_t) - \nabla f_n(\y^n_{t,i})}} - \frac{\gamma}{2}\Expectbracket{\normsq{\nabla f(\x_t)}}\\
    &\leq  \frac{\gamma L^2}{2N}\sum_{n=1}^N\Expectbracket{\normsq{\x_t - \y^n_{t,i}}} - \frac{\gamma}{2}\Expectbracket{\normsq{\nabla f(\x_t)}}\\
    &\leq \frac{\gamma L^2}{2N}\sum_{n=1}^N\Expectbracket{\normsq{\sum_{j=0}^{i-1} \gamma\g_n(\y^n_{t,j})}} - \frac{\gamma}{2}\Expectbracket{\normsq{\nabla f(\x_t)}}\\
    &\leq \frac{\gamma^3 L^2(I-1)}{2N}\sum_{n=1}^N\sum_{j=0}^{i-1}\Expectbracket{\normsq{ \g_n(\y^n_{t,j})}} - \frac{\gamma}{2}\Expectbracket{\normsq{\nabla f(\x_t)}}
\end{align*}

Plugging back to (\ref{eq:convergence_proof_2}), we have
\begin{align}
    &\Expectbracket{f(\x_{t+1})} \nonumber\\
    &\leq \Expectbracket{f(\x_t)}  + \frac{\gamma^3 L^2(I-1)}{2N}\sum_{n=1}^N\sum_{i=0}^{I-1}\sum_{j=0}^{i-1}\Expectbracket{\normsq{ \g_n(\y^n_{t,j})}} \nonumber\\
    &\qquad - \frac{\gamma I}{2}\Expectbracket{\normsq{\nabla f(\x_t)}} \nonumber\\
    &\qquad + \frac{LI \gamma^2}{2N}  \sum_{n=1}^N \frac{1}{q_n^t} \sum_{i=0}^{I-1}  \Expectbracket{\normsq{  \g_n(\y^n_{t,i})}} \label{eq:convergence_proof_3}
\end{align}

Rearranging and summing $t$ from $0$ to $T-1$, we have
\begin{align}
    &\frac{1}{T}\sum_{t=0}^{T-1}\Expectbracket{\normsq{\nabla f(\x_t)}} \nonumber\\
    &\quad\leq \frac{2\left(\Expectbracket{f(\x_0)} - \Expectbracket{f(\x_{T})}\right)}{\gamma TI}\nonumber\\
    &\qquad+ \frac{\gamma^2 L^2(I-1)}{TIN}\sum_{t=0}^{T-1}\sum_{n=1}^N\sum_{i=0}^{I-1}\sum_{j=0}^{i-1}\Expectbracket{\normsq{ \g_n(\y^n_{t,j})}} \nonumber\\
    &\qquad + \frac{\gamma L}{TN}  \sum_{t=0}^{T-1}\sum_{n=1}^N \frac{1}{q_n^t} \sum_{i=0}^{I-1}  \Expectbracket{\normsq{  \g_n(\y^n_{t,i})}} \nonumber\\
    &\quad\leq \frac{2\left(f(\x_0) - f^*\right)}{\gamma TI}\nonumber\\
    &\qquad + \frac{\gamma^2 L^2(I-1)}{TIN}\sum_{t=0}^{T-1}\sum_{n=1}^N\sum_{i=0}^{I-1}\sum_{j=0}^{i-1}\Expectbracket{\normsq{ \g_n(\y^n_{t,j})}} \nonumber\\
    &\qquad + \frac{\gamma L}{TN}  \sum_{t=0}^{T-1}\sum_{n=1}^N \frac{1}{q_n^t} \sum_{i=0}^{I-1}  \Expectbracket{\normsq{  \g_n(\y^n_{t,i})}}
\end{align}

\subsection{Proof of Theorem \ref{thm:Lyapunov}}\label{sec:appendix2}
Since there are only two variables to solve for and two simple boundary constraints per $n$, we can find the minimizing values of $q_n^t$ and $P_n(t)$ by finding the roots of the gradient of the objective function and ensuring that they are within the upper and lower bounds.
If no roots are within that set, one of the end points will minimize the function, so we only need to check those points.

To find the roots, we compute the gradient of the objective function for each $n$ in \eqref{eqn:minimization}
\begin{align}\label{eqn:gradient}
    &\nabla f(q_n^t,P_n(t)) = \nonumber \\
    &\begin{bmatrix}
    -\frac{V}{N (q_n^t)^2} + \frac{ V\lambda\ell}{B\log_2\left(1+|h_n(t)|^2 \frac{P_n(t)}{N_0}\right) }+Z_n(t)P_n(t)\\
     \frac{-V\lambda\ell|h_n(t)|^2}{N_0 B\left(1+|h_n(t)|^2 \frac{P_n(t)}{N_0}\right)\left(\log_2\left(1+|h_n(t)|^2 \frac{P_n(t)}{N_0}\right)\right)^2} q_n^t +Z_n(t) q_n^t
    \end{bmatrix}.
\end{align}
We first look at the partial derivative with respect to $P_n(t)$ and note that setting it equal to zero and dividing by $q_n^t$ gives
\begin{align*}
    0=&\nonumber\\
    &\frac{-V\lambda\ell |h_n(t)|^2/(N_0 B)}{\left(1+|h_n(t)|^2 \frac{P_n(t)}{N_0}\right)\left(\log_2\left(1+|h_n(t)|^2 \frac{P_n(t)}{N_0}\right)\right)^2} +Z_n(t) 
\end{align*}
which does not depend on $q_n^t$.
Next, let $A=\frac{V\lambda\ell |h_n(t)|^2 \left(\log(2)\right)^2}{N_0 B Z_n(t)}$ and $x=1+|h_n(t)|^2 \frac{P_n(t)}{N_0}$, then we have something in the form of
\begin{align*}
    A &= x \left(\log(x)\right)^2 \nonumber  = x \left(\log\left(\sfrac{1}{x}\right)\right)^2 .
\end{align*}
By dividing both sides by $1/4$, letting $x' = \sqrt{\frac{A}{4}} \frac{1}{\sqrt{x}}$, and rearranging, we have
\begin{align*}
    \sqrt{\sfrac{A}{4}}=x' e^{x'}
\end{align*}
that has a known solution of $ x' =W_k \left(\sqrt{\frac{A}{4}}\right)$ where $W_k(\cdot)$ is the Lambert $W$ function which solves $w\exp{w}=z$ for $w$.

To get the critical point for $P_n(t)$, we unwrap and substitute $P_n(t) = \frac{N_0}{|h_n(t)|^2}(x-1)$, to get
\begin{align} \label{eqn:powerOpt2}
    P_n^\textnormal{opt}(t) = \frac{N_0}{|h_n(t)|^2} \left(\frac{A}{4} W_k\left(\sqrt{\frac{A}{4}}\right)^{-2}-1\right)
\end{align}
which has a single root at $k=0$ since $\sqrt{\frac{A}{4}}\geq 0$.

Finally, for the critical point for $q_n^t$, we can plug $P_n^\textnormal{opt}(t)$ into the partial derivative with respect to $q_n^t$ to get \eqref{eqn:qOpt}.

\clearpage

\bibliographystyle{./bibliography/IEEEtran}
\bibliography{references.bib}

\begin{thebibliography}{10}
\providecommand{\url}[1]{#1}
\csname url@samestyle\endcsname
\providecommand{\newblock}{\relax}
\providecommand{\bibinfo}[2]{#2}
\providecommand{\BIBentrySTDinterwordspacing}{\spaceskip=0pt\relax}
\providecommand{\BIBentryALTinterwordstretchfactor}{4}
\providecommand{\BIBentryALTinterwordspacing}{\spaceskip=\fontdimen2\font plus
\BIBentryALTinterwordstretchfactor\fontdimen3\font minus
  \fontdimen4\font\relax}
\providecommand{\BIBforeignlanguage}[2]{{%
\expandafter\ifx\csname l@#1\endcsname\relax
\typeout{** WARNING: IEEEtran.bst: No hyphenation pattern has been}%
\typeout{** loaded for the language `#1'. Using the pattern for}%
\typeout{** the default language instead.}%
\else
\language=\csname l@#1\endcsname
\fi
#2}}
\providecommand{\BIBdecl}{\relax}
\BIBdecl

\bibitem{mcmahan2017communication}
B.~McMahan, E.~Moore, D.~Ramage, S.~Hampson, and B.~A. y~Arcas,
  ``Communication-efficient learning of deep networks from decentralized
  data,'' in \emph{Artificial intelligence and statistics}.\hskip 1em plus
  0.5em minus 0.4em\relax PMLR, 2017, pp. 1273--1282.

\bibitem{li2019convergence}
X.~Li, K.~Huang, W.~Yang, S.~Wang, and Z.~Zhang, ``On the convergence of
  {FedAvg} on non-iid data,'' \emph{arXiv preprint arXiv:1907.02189}, 2019.

\bibitem{mitra2021achieving}
A.~Mitra, R.~Jaafar, G.~J. Pappas, and H.~Hassani, ``Achieving linear
  convergence in federated learning under objective and systems
  heterogeneity,'' \emph{arXiv preprint arXiv:2102.07053}, 2021.

\bibitem{mothukuri2021survey}
V.~Mothukuri, R.~M. Parizi, S.~Pouriyeh, Y.~Huang, A.~Dehghantanha, and
  G.~Srivastava, ``A survey on security and privacy of federated learning,''
  \emph{Future Generation Computer Systems}, vol. 115, pp. 619--640, 2021.

\bibitem{geyer2017differentially}
R.~C. Geyer, T.~Klein, and M.~Nabi, ``Differentially private federated
  learning: A client level perspective,'' \emph{arXiv preprint
  arXiv:1712.07557}, 2017.

\bibitem{truex2019hybrid}
S.~Truex, N.~Baracaldo, A.~Anwar, T.~Steinke, H.~Ludwig, R.~Zhang, and Y.~Zhou,
  ``A hybrid approach to privacy-preserving federated learning,'' in
  \emph{Proceedings of the 12th ACM Workshop on Artificial Intelligence and
  Security}, 2019, pp. 1--11.

\bibitem{li2020federated}
T.~Li, A.~K. Sahu, M.~Zaheer, M.~Sanjabi, A.~Talwalkar, and V.~Smith,
  ``Federated optimization in heterogeneous networks,'' \emph{Proceedings of
  Machine Learning and Systems}, vol.~2, pp. 429--450, 2020.

\bibitem{wang2019adaptive}
S.~Wang, T.~Tuor, T.~Salonidis, K.~K. Leung, C.~Makaya, T.~He, and K.~Chan,
  ``Adaptive federated learning in resource constrained edge computing
  systems,'' \emph{IEEE Journal on Selected Areas in Communications}, vol.~37,
  no.~6, pp. 1205--1221, 2019.

\bibitem{konevcny2016federated}
J.~Kone{\v{c}}n{\`y}, H.~B. McMahan, F.~X. Yu, P.~Richt{\'a}rik, A.~T. Suresh,
  and D.~Bacon, ``Federated learning: Strategies for improving communication
  efficiency,'' \emph{arXiv preprint arXiv:1610.05492}, 2016.

\bibitem{han2020adaptive}
P.~Han, S.~Wang, and K.~K. Leung, ``Adaptive gradient sparsification for
  efficient federated learning: An online learning approach,'' in \emph{2020
  IEEE 40th International Conference on Distributed Computing Systems
  (ICDCS)}.\hskip 1em plus 0.5em minus 0.4em\relax IEEE, 2020, pp. 300--310.

\bibitem{sattler2019robust}
F.~Sattler, S.~Wiedemann, K.-R. M{\"u}ller, and W.~Samek, ``Robust and
  communication-efficient federated learning from non-iid data,'' \emph{IEEE
  transactions on neural networks and learning systems}, vol.~31, no.~9, pp.
  3400--3413, 2019.

\bibitem{alistarh2017qsgd}
D.~Alistarh, D.~Grubic, J.~Li, R.~Tomioka, and M.~Vojnovic, ``Qsgd:
  Communication-efficient sgd via gradient quantization and encoding,''
  \emph{Advances in Neural Information Processing Systems}, vol.~30, pp.
  1709--1720, 2017.

\bibitem{albasyoni2020optimal}
A.~Albasyoni, M.~Safaryan, L.~Condat, and P.~Richt{\'a}rik, ``Optimal gradient
  compression for distributed and federated learning,'' \emph{arXiv preprint
  arXiv:2010.03246}, 2020.

\bibitem{nishio2019client}
T.~Nishio and R.~Yonetani, ``Client selection for federated learning with
  heterogeneous resources in mobile edge,'' in \emph{ICC 2019-2019 IEEE
  International Conference on Communications (ICC)}.\hskip 1em plus 0.5em minus
  0.4em\relax IEEE, 2019, pp. 1--7.

\bibitem{ribero2020communication}
M.~Ribero and H.~Vikalo, ``Communication-efficient federated learning via
  optimal client sampling,'' \emph{arXiv preprint arXiv:2007.15197}, 2020.

\bibitem{goetz2019active}
J.~Goetz, K.~Malik, D.~Bui, S.~Moon, H.~Liu, and A.~Kumar, ``Active federated
  learning,'' \emph{arXiv preprint arXiv:1909.12641}, 2019.

\bibitem{yang2019scheduling}
H.~H. Yang, Z.~Liu, T.~Q. Quek, and H.~V. Poor, ``Scheduling policies for
  federated learning in wireless networks,'' \emph{IEEE Transactions on
  Communications}, vol.~68, no.~1, pp. 317--333, 2019.

\bibitem{cho2020client}
Y.~J. Cho, J.~Wang, and G.~Joshi, ``Client selection in federated learning:
  Convergence analysis and power-of-choice selection strategies,'' \emph{arXiv
  preprint arXiv:2010.01243}, 2020.

\bibitem{ren2020scheduling}
J.~Ren, Y.~He, D.~Wen, G.~Yu, K.~Huang, and D.~Guo, ``Scheduling for cellular
  federated edge learning with importance and channel awareness,'' \emph{IEEE
  Transactions on Wireless Communications}, vol.~19, no.~11, pp. 7690--7703,
  2020.

\bibitem{ruan2021towards}
Y.~Ruan, X.~Zhang, S.-C. Liang, and C.~Joe-Wong, ``Towards flexible device
  participation in federated learning,'' in \emph{International Conference on
  Artificial Intelligence and Statistics}.\hskip 1em plus 0.5em minus
  0.4em\relax PMLR, 2021, pp. 3403--3411.

\bibitem{karimireddy2020scaffold}
S.~P. Karimireddy, S.~Kale, M.~Mohri, S.~Reddi, S.~Stich, and A.~T. Suresh,
  ``Scaffold: Stochastic controlled averaging for federated learning,'' in
  \emph{International Conference on Machine Learning}.\hskip 1em plus 0.5em
  minus 0.4em\relax PMLR, 2020, pp. 5132--5143.

\bibitem{yang2021achieving}
H.~Yang, M.~Fang, and J.~Liu, ``Achieving linear speedup with partial worker
  participation in non-iid federated learning,'' \emph{arXiv preprint
  arXiv:2101.11203}, 2021.

\bibitem{gu2021fast}
X.~Gu, K.~Huang, J.~Zhang, and L.~Huang, ``Fast federated learning in the
  presence of arbitrary device unavailability,'' \emph{arXiv preprint
  arXiv:2106.04159}, 2021.

\bibitem{chen2020joint}
M.~Chen, Z.~Yang, W.~Saad, C.~Yin, H.~V. Poor, and S.~Cui, ``A joint learning
  and communications framework for federated learning over wireless networks,''
  \emph{IEEE Transactions on Wireless Communications}, vol.~20, no.~1, pp.
  269--283, 2020.

\bibitem{yang2020energy}
Z.~Yang, M.~Chen, W.~Saad, C.~S. Hong, and M.~Shikh-Bahaei, ``Energy efficient
  federated learning over wireless communication networks,'' \emph{IEEE
  Transactions on Wireless Communications}, 2020.

\bibitem{wadu2020federated}
M.~M. Wadu, S.~Samarakoon, and M.~Bennis, ``Federated learning under channel
  uncertainty: Joint client scheduling and resource allocation,'' in \emph{2020
  IEEE Wireless Communications and Networking Conference (WCNC)}.\hskip 1em
  plus 0.5em minus 0.4em\relax IEEE, 2020, pp. 1--6.

\bibitem{huang2020efficiency}
T.~Huang, W.~Lin, W.~Wu, L.~He, K.~Li, and A.~Y. Zomaya, ``An
  efficiency-boosting client selection scheme for federated learning with
  fairness guarantee,'' \emph{IEEE Transactions on Parallel and Distributed
  Systems}, vol.~32, no.~7, pp. 1552--1564, 2020.

\bibitem{zhou2020cefl}
Z.~Zhou, S.~Yang, L.~Pu, and S.~Yu, ``Cefl: online admission control, data
  scheduling, and accuracy tuning for cost-efficient federated learning across
  edge nodes,'' \emph{IEEE Internet of Things Journal}, vol.~7, no.~10, pp.
  9341--9356, 2020.

\bibitem{neely2010stochastic}
M.~J. Neely, ``Stochastic network optimization with application to
  communication and queueing systems,'' \emph{Synthesis Lectures on
  Communication Networks}, vol.~3, no.~1, pp. 1--211, 2010.

\bibitem{krizhevsky2009learning}
A.~Krizhevsky and G.~Hinton, ``Learning multiple layers of features from tiny
  images,'' University of Toronto, Toronto, Ontario, Tech. Rep.~0, 2009.

\bibitem{caldas2018leaf}
S.~Caldas, S.~M.~K. Duddu, P.~Wu, T.~Li, J.~Kone{\v{c}}n{\`y}, H.~B. McMahan,
  V.~Smith, and A.~Talwalkar, ``Leaf: A benchmark for federated settings,''
  \emph{arXiv preprint arXiv:1812.01097}, 2018.

\bibitem{cohen2017emnist}
G.~Cohen, S.~Afshar, J.~Tapson, and A.~Van~Schaik, ``{EMNIST}: Extending
  {MNIST} to handwritten letters,'' in \emph{2017 International Joint
  Conference on Neural Networks (IJCNN)}.\hskip 1em plus 0.5em minus
  0.4em\relax IEEE, 2017, pp. 2921--2926.

\bibitem{zhao2018federated}
Y.~Zhao, M.~Li, L.~Lai, N.~Suda, D.~Civin, and V.~Chandra, ``Federated learning
  with non-iid data,'' \emph{arXiv preprint arXiv:1806.00582}, 2018.

\end{thebibliography}

\end{document}